\documentclass{article}

\newif\ifarxiv
\arxivfalse

\newif\ifneurips
\ifarxiv
	\neuripsfalse
\else
	\neuripstrue
\fi	

\ifarxiv
	\usepackage[preprint]{neurips_2020}
\else
	\usepackage[final]{neurips_2020}
\fi

\usepackage[utf8]{inputenc} 
\usepackage[T1]{fontenc}    
\usepackage{hyperref}       
\usepackage{url}            
\usepackage{booktabs}       
\usepackage{amsfonts}       
\usepackage{nicefrac}       
\usepackage{microtype}      
\usepackage{amsmath,amssymb,amsthm,color, amsfonts}
\usepackage{algorithm}
\usepackage{algorithmic}
\usepackage{xspace}
\usepackage{graphicx}
\usepackage{lipsum}  
\usepackage{xcolor}
\title{Approximate Cross-Validation for Structured Models}

\usepackage{mathtools}

\DeclarePairedDelimiter\abs{\lvert}{\rvert}

\usepackage[capitalize]{cleveref}
\Crefname{lm}{Lemma}{Lemmas}
\Crefname{prop}{Proposition}{Propositions}
\Crefname{defn}{Definition}{Definitions}
\Crefname{thm}{Theorem}{Theorems}
\Crefname{assumption}{Assumption}{Assumptions}
\Crefname{cor}{Corollary}{Corollaries}
\Crefname{condition}{Condition}{Conditions}

\author{%
  Soumya Ghosh\thanks{Joint first authorship}\hspace{5pt}\thanks{Also with the Center for Computational Health}\\
  MIT-IBM Watson AI Lab \\
  IBM Research \\
  \texttt{ghoshso@us.ibm.com} \\
  \And
  William T. Stephenson\footnotemark[1] \\
  MIT CSAIL \\ MIT-IBM Watson AI Lab \\
  \texttt{wtstephe@mit.edu} \\
  \And 
  Tin D. Nguyen \\
  MIT CSAIL \\ MIT-IBM Watson AI Lab \\
  \texttt{tdn@mit.edu} \\
  \And 
  Sameer K. Deshpande \\
  MIT CSAIL \\ MIT-IBM Watson AI Lab \\
  \texttt{sameerd@alum.mit.edu} \\
  \And
  Tamara Broderick \\
  MIT CSAIL \\ MIT-IBM Watson AI Lab \\
  \texttt{tbroderick@csail.mit.edu}
}

\begin{document}
\newcommand{\param}{\Theta}
\newcommand{\paramone}{\hat{\Theta}(\mathbf{1}_T)}
\newcommand{\paramonent}{\hat{\Theta}(\mathbf{1}_{NT})}
\newcommand{\paramoneseq}{\hat{\Theta}(\mathbf{1}_T)}
\newcommand{\paramoneseqpartial}[1]{\hat{\Theta}(\mathbf{1}_{#1})}
\newcommand{\paramperturb}[1]{\hat{\Theta}(#1)}
\newcommand{\paramperturbij}[1]{\hat{\Theta}_{IJ}(#1)}
\newcommand{\paramperturbapprox}[1]{{\tilde{\Theta}}(#1)}
\newcommand{\paramoneapprox}{{\tilde{\Theta}}(\mathbf{1})}
\newcommand{\paramloo}{\hat{\Theta}^{\backslash n}}
\newcommand{\paramlo}{\hat{\Theta}^{\backslash \mathbf{o}}}
\newcommand{\paramlooij}{\hat{\Theta}_{\textrm{IJ}}^{\backslash n}}
\newcommand{\paramij}{\hat{\Theta}_{\textrm{IJ}}} 
\newcommand{\paramstep}[1]{\Theta^{(#1)}} 
\newcommand{\paramalg}{\hat\Theta_{\textrm{ACV}}} 
\newcommand{\xnew}{x_{\text{new}}}
\newcommand{\xnewt}[1]{x_{\text{new }{#1}}}
\newcommand{\lapprox}{\tilde{l}}
\newcommand{\eye}{\mathbf{I}}
\newcommand{\normal}{\mathcal{N}}
\newcommand{\ent}[1]{\ensuremath{\mathbb{H}\big[ #1 \big]}}
\newcommand{\E}[1]{\mathbb{E}\big[ #1 \big]}
\newcommand{\Ewrt}[2]{\mathbb{E}_{#1}\big[ #2 \big]}
\newcommand{\xseq}{\mathbf{x}_n}
\newcommand{\xseqone}{\mathbf{x}}
\newcommand{\zseq}{\mathbf{z}_n}
\newcommand{\zseqone}{\mathbf{z}}
\newcommand{\zobs}{\zseqone_{\mathbf{o}}}
\newcommand{\xseqt}[1]{x_{n{#1}}}
\newcommand{\xseqonet}[1]{x_{{#1}}}
\newcommand{\zseqt}[1]{z_{n{#1}}}
\newcommand{\zseqonet}[1]{z_{{#1}}}
\newcommand{\seqwt}{s_n}
\newcommand{\wseq}[1]{\mathbf{w}_{#1}}
\newcommand{\wseqone}{\mathbf{w}}
\newcommand{\wseqo}{\mathbf{w_{o}}}
\newcommand{\wt}{w_t}
\newcommand{\wtt}[1]{w_{#1}}
\newcommand{\xobs}{\mathbf{x}_{[T] - \mathbf{o}}}
\newcommand{\xmiss}{\mathbf{x}_{\mathbf{o}}}
\newcommand{\argmin}[1]{\underset{#1}{\text{argmin }}}
\newcommand{\factors}{\mathcal{F}} 
\newcommand{\leftoutinds}{\mathbf{o}} 
\newcommand{\allfolds}{\mathcal{O}}
\newcommand{\numfolds}{\abs{\mathcal{O}}}
\newcommand{\data}{\mathcal{D}}

\newcommand{\loss}{\mathcal{L}}
\newcommand{\losscv}{\mathcal{L}_{\mathrm{CV}}} 
\newcommand{\losscvexact}{\mathcal{L}_{\mathrm{eCV}}} 
\newcommand{\lossij}{\mathcal{L}_{\mathrm{IJ}}} 
\newcommand{\lossinexact}{\mathcal{L}_{\mathrm{ACV}}} 

\newcommand{\epsparam}{\eps_\param}
\newcommand{\epsij}{\eps_{\mathrm{IJ}}}
\newcommand{\numseq}{N} 
\newcommand{\dimdata}{R} 
\newcommand{\numsteps}{T} 
\newcommand{\numparams}{D} 
\newcommand{\numhmmstates}{K} 

\newcommand{\reals}{\mathbb{R}}
\newcommand{\posreals}{\mathbb{R}_+}
\newcommand{\simplex}{\Delta}

\newcommand{\lwcv}{LWCV\xspace} 
\newcommand{\lscv}{LSCV\xspace} 

\newcommand{\cat}{\textrm{Cat}} 

\newcommand{\err}{\mathrm{Err}}
\newcommand{\todo}[1]{\textcolor{red}{TODO: #1}}
\newcommand{\eps}{\varepsilon}
\newcommand{\R}{\mathbb{R}}
\newcommand{\at}[1]{\Bigr|_{\substack{#1}}}
\newcommand{\n}[1]{\left\lVert #1 \right\rVert}
\newcommand{\dnone}{D_n^{(1)}} 

\newtheorem{thm}{Theorem}
\newtheorem{prop}[]{Proposition}
\newtheorem{lm}[]{Lemma}
\newtheorem{defn}[]{Definition}
\newtheorem{rem}[]{Remark}
\newtheorem{eg}[]{Example}
\newtheorem{cl}{Claim}[thm]
\newtheorem{cor}[]{Corollary}
\newtheorem{assumption}[]{Assumption}
\newtheorem{condition}[]{Condition}

\maketitle
\begin{abstract}
Many modern data analyses benefit from explicitly modeling dependence structure in data -- such as measurements across time or space, ordered words in a sentence, or genes in a genome. A gold standard evaluation technique is structured cross-validation (CV), which leaves out some data subset (such as data within a time interval or data in a geographic region) in each fold. But CV here can be prohibitively slow due to the need to re-run already-expensive learning algorithms many times. 
Previous work has shown approximate cross-validation (ACV) methods provide a fast and provably accurate alternative in the setting of empirical risk minimization.
But this existing ACV work is restricted to simpler models by the assumptions that (i) data across CV folds are independent and (ii) an exact initial model fit is available. In structured data analyses, both these assumptions are often untrue. 
In the present work, we address (i) by extending ACV to CV schemes with dependence structure between the folds. 
To address (ii), we verify -- both theoretically and empirically -- that ACV quality deteriorates smoothly with noise in the initial fit. We demonstrate the accuracy and computational benefits of our proposed methods on a diverse set of real-world applications.

\end{abstract}

\section{Introduction}

Models with complex dependency structures have become standard machine learning tools in analyses of data from science, social science, and engineering fields. These models are used to characterize disease progression \citep{sukkar12, wang14, sun19}, to track crime in a city \citep{BalocchiJensen2019, Balocchi2020}, and to monitor and potentially manage traffic flow \citep{ihler2006adaptive,zheng2017estimating} 
among many other applications.
The potential societal impact of these methods necessitates that they be used and evaluated with care. Indeed, recent work \citep{musgrave2020metric}
has emphasized that hyperparameter tuning and assessment with \emph{cross-validation} (CV) \citep{stone1974earlyCV, geisser1975earlyCV} is crucial to trustworthy and meaningful analysis of modern, complex machine learning methods.

While CV offers a conceptually simple and widely used tool for evaluation, it can be computationally prohibitive in complex models.
These models often already face severe computational demands to fit just once, and CV requires multiple re-fits. 
To address this cost, recent authors \citep{beirami2017firstALOO, rad2018detailedALOO, giordano2018swiss} have proposed \emph{approximate} CV (ACV) methods; their work demonstrates that ACV methods perform well in both theory and practice for a collection of practical models. 
These methods take two principal forms: one approximation based on a Newton step (NS) \citep{beirami2017firstALOO, rad2018detailedALOO} and one based on the classical infinitesimal jackknife (IJ) from statistics \citep{koh2017understanding, beirami2017firstALOO, giordano2018swiss}. Though both ACV forms show promise, there remain major roadblocks to applying either NS or IJ to models with dependency structure. First, all existing ACV theory and algorithms assume that data dropped out by each CV fold are independent of the data in the other folds.
But to evaluate time series models, for instance, we often drop out data points in various segments of time. Or we might drop out data within a geographic region to evaluate a spatiotemporal model. In all of these cases, the independence assumption would not apply.
Second, NS methods require recomputation and inversion of a model's Hessian matrix at each CV fold. In the complex models we consider here, this cost can itself be prohibitive. 
Finally, existing theory for IJ methods requires an exact initial fit of the model -- and authors so far have taken great care to obtain such a fit \citep{giordano2018swiss, stephenson2019sparse}. But practitioners learning in e.g.\ large sequences or graphs typically settle for an approximate fit to limit computational cost.

In this paper, we address these concerns and thereby expand the reach of ACV to include more sophisticated models with dependencies among data points and for which exact model fits are infeasible. 
To avoid the cost of matrix recomputation and inversion across folds, we here focus on the IJ, rather than the NS. In particular, in \cref{sec:structured_IJ}, we develop IJ approximations for dropping out individual nodes in a dependence graph. Our methods allow us e.g.\ to leave out points within, or at the end of, a time series -- but our methods also apply to more general Markov random fields, without a strict chain structure. 
In \cref{sec:inexact_optimization}, we demonstrate that the IJ yields a useful ACV method even without an exact initial model fit. In fact, we show that the quality of the IJ approximation decays with the quality of the initial fit in a smooth and interpretable manner. Finally, we demonstrate our method on a diverse set of real-world applications and models in \cref{sec:exp,app:extra_experiments}. These include count data analysis with time-varying Poisson processes, named entity recognition with neural conditional random fields, motion capture analysis with auto-regressive hidden Markov models, and a spatial analysis of crime data with hidden Markov random fields.

\section{Structured models and cross-validation}
\label{sec:background}

\subsection{Structured models}
Throughout we consider two types of models: (1) hidden Markov random fields (MRFs) with observations $\xseqone$ and latent variables $\zseqone$ and (2) conditional random fields (CRFs) with inputs (i.e., covariates) $\xseqone$ and labels $\zseqone$, both observed. Our developments for hidden MRFs and CRFs are very similar, but with slight differences. We detail MRFs in the main text; throughout, we will refer the reader to the appendix for the CRF treatment. We first give an illustrative example of MRFs and then the general formulation; a CRF overview appears in \cref{sec:supp_crf}.

\textbf{Example}: Hidden Markov Models (HMMs) capture sequences of observations such as words in a sentence or longitudinally measured physiological signals. Consider an HMM with $\numseq$ (independent) sequences, $\numsteps$ time steps, and $\numhmmstates$ states. We take each observation to have dimension $\dimdata$. So the $t$th observed element in the $n$th sequence is $\xseqt{t} \in \reals^{\dimdata}$, and the latent $\zseqt{t} \in [K] := \{1,\ldots, K\}$.
The model is specified by (1) a distribution on the initial latent state $p(\zseqt{1}) = \cat(\zseqt{1} \mid \pi)$, where $\cat$ is the categorical distribution and $\pi \in \simplex_{\numhmmstates-1}$, the $\numhmmstates-1$ simplex; (2) a $\numhmmstates \times \numhmmstates$ transition matrix $A$ with columns $A_{k} \in \simplex_{\numhmmstates-1}$ and $p(\zseqt{t} \mid \zseqt{,t-1}) = \cat(\zseqt{t} \mid A_{\zseqt{,t-1}})$; and (3) emission distributions $F$ with parameters $\theta_k$ such that $p(\xseqt{t} \mid \zseqt{t}) = F(\xseqt{t} \mid \theta_{\zseqt{t}})$.
We collect all parameters of the model in $\param := \{\pi, \{A_k\}_{k=1}^K, \{\theta_k\}_{k=1}^K \}$. We consider $\param$ as a vector of length $\numparams$. We may have a prior $p(\param)$.

More generally, we consider \textbf{(hidden) MRFs} with $\numseq$ structured observations $\xseq$ and latents $\zseq$, independent across $n \in [\numseq]$. We index single observations of dimension $\dimdata$ (respectively, latents) within the structure by $t \in [\numsteps]$: $\xseqt{t} \in \reals^{\dimdata}$ (respectively, $\zseqt{t}$). 
Our experiments will focus on bounded, discrete $\zseqt{t}$ (i.e., $\zseqt{t} \in [K]$), but we use more inclusive notation (that might e.g.\ apply to continuous latents) when possible.
We consider models with parameters $\param \in \reals^{\numparams}$ and a single emission factor for each latent.
\begin{equation}
	\label{eq:mrf}
	-\log p(\xseqone, \zseqone; \param )
		= 
		Z(\param) +
		\sum_{n=1}^{N}
		\left\{ \left[
    			\sum_{t \in [T]} \psi_{t}(\xseqt{t}, \zseqt{t}; \param)
    			\right]
    		+ \left[ 
    			\sum_{\textbf{c} \in \factors} \phi_{\textbf{c}}( \zseqt{\textbf{c}}; \param)
    			\right]
		\right\},
\end{equation}
where $\zseqt{\textbf{c}} := (\zseqt{t})_{t \in \textbf{c}}$ for $\textbf{c} \subseteq [\numsteps]$; $\psi_{t}$ is a log factor mapping $(\xseqt{t}, \zseqt{t})$ to $\reals$; $\phi_{\textbf{c}}$ is a log factor mapping collections of latents, indexed by $\textbf{c}$, to $\reals$; $\factors$ collects the subsets indexing factors; and $Z(\param)$ is a negative log normalizing constant. HMMs, as described above, are a special case; see \cref{sec:supp_examples} for details.
For any MRF, we can learn the parameters by marginalizing the latents and maximizing the posterior, or equivalently the joint, in $\param$.
Maximum likelihood estimation is the special case with formal prior $p(\param)$ constant across $\param$.
\begin{equation}
	\label{eq:fit_structure}
	\hat\param := \argmin{\param}  -\log p(\xseqone; \param) - \log p(\param) = \argmin{\param} -\log \int_{\zseqone} p(\xseqone, \zseqone; \param) \; d\zseqone - \log p(\param).
\end{equation}
\subsection{Challenges of cross-validation and approximate cross-validation in structured models}

In CV procedures, we iteratively leave out some data in order to diagnose variation in $\hat\param$ under natural data variability or to estimate the predictive accuracy of our model. We consider two types of CV of interest in structured models; we make these formulations precise later. (1) We say that we consider \emph{leave-within-structure-out} CV (\lwcv) when we remove some data points $\xseqt{t}$ within a structure and learn on the remaining data points. For instance, we might try to predict crime in certain census tracts based on observations in other tracts. Often in this case $\numseq = 1$~\citep{celeux2008selecting}, ~\citep[Chapter 3.4]{hyndman2018forecasting}, and we assume \lwcv has $\numseq = 1$ for notational simplicity in what follows. (2) We say that we consider \emph{leave-structure-out} CV (\lscv) when we leave out entire $\xseq$ for either a single $n$ or a collection of $n$. For instance, with a state-space model of gene expression, we might predict some individuals' gene expression profiles given other individuals' profiles. In this case, $\numseq \gg 1$ \citep{rangel2004modeling,decaprio2007conrad}. In either (1) or (2), the goal of CV is to consider multiple folds, or subsets of data, left out to assess variability and improve estimation of held-out error. But every fold incurs the cost of the learning procedure in \cref{eq:fit_structure}. Indeed, practitioners have explicitly noted the high cost of using multiple folds and have resorted to using only a few, large folds \citep{celeux2008selecting}, leading to biased or noisy estimates of the out-of-sample variability.

A number of researchers have addressed the prohibitive cost of CV with approximate CV (ACV) procedures for simpler models \citep{beirami2017firstALOO, rad2018detailedALOO, giordano2018swiss}. Existing work focuses on the following learning problem with weights $\wseqone \in \R^J$:
\begin{equation}
	\hat\param(\wseqone) = \argmin{\param} \sum_{j \in [J]} \wtt{j} f_j(\param) + \lambda R(\param),
	\label{eq:old_approx_cv}	
\end{equation}
where $\forall j \in [J], f_j, R: \reals^{\numparams} \rightarrow \reals$ and $\lambda \in \reals_+$. 
When the weight vector $\wseqone$ equals the all-ones vector $\mathbf{1}_{J}$, we recover a regularized empirical loss minimization problem. By considering all weight vectors with one weight equal to zero, we recover the folds of leave-one-out CV; other forms of CV can be similarly recovered. 
The notation $\hat\param(\wseqone)$ emphasizes that the learned parameter values depend on the weights.

To see if this framework applies to \lwcv or \lscv, we can interpret $f_j$ as a negative log likelihood (up to normalization) for the $j$th data point and $\lambda R$ as a negative log prior. Then the likelihood corresponding to the objective of \cref{eq:old_approx_cv} factorizes as $p(\xseqone \mid \param) = \prod_{j \in J} p(\xseqonet{j} \mid \param) \propto \prod_{j \in J} \exp(-f_t(\param))$. This factorization amounts to an independence assumption across the $\{\xseqonet{j}\}_{j \in [J]}$. 
In the case of \lwcv, with $\numseq = 1$, $j$ must serve the role of $t$, and $J = \numsteps$. But the $x_t$ are not independent, so we cannot apply existing ACV methods.
    In the \lscv case, $\numseq \ge 1$, and $j$ in \cref{eq:old_approx_cv} can be seen as serving the role of $n$, with $J = \numseq$. Since the $\xseq$ are independent, \cref{eq:old_approx_cv} can express \lscv folds. 
  
\textbf{Previous ACV work} provides two primary options for the \lscv case. We give a brief review here, but see \cref{app:related_work} for a more detailed review. One option is based on taking a single Newton step on the \lscv objective starting from $\hat\param(\mathbf{1}_{\numsteps})$ \citep{beirami2017firstALOO, rad2018detailedALOO}. Except in special cases -- such as leave-one-out CV for generalized linear models -- this Newton-step approach requires both computing and inverting a new Hessian matrix for each fold, often a prohibitive expense; see \cref{app:NS} for a discussion. 
An alternative method \citep{koh2017understanding, beirami2017firstALOO, giordano2018swiss} based on the \emph{infinitesimal jackknife} (IJ) from statistics \citep{jaeckel1972infinitesimal, efron1981nonparametric} constructs a Taylor expansion of $\hat\param(\wseqone)$ around $\wseqone = \mathbf{1}_\numsteps$. 
For any model of the form in \cref{eq:old_approx_cv}, the IJ requires just a single Hessian matrix computation and inversion.
Therefore, we focus on the IJ for \lscv and use the IJ for inspiration when developing \lwcv below. However, all existing IJ theory and empirics require access to an \emph{exact} minimum for $\hat\param(\mathbf{1}_{J})$. Indeed, previous authors~\citep{giordano2018swiss, stephenson2019sparse} have taken great care to find an exact minimum of \cref{eq:old_approx_cv}. Unfortunately, for most complex, structured models with large datasets, finding an exact minimum requires an impractical amount of computation. Others~\citep{burkner2019approximate} have developed ACV methods for Bayesian time series models and for Bayesian models without dependence structures \citep{vehtari2017practical}. Our development here focuses on empirical risk minimization and is not restricted to temporal models.

In the following, we extend the reach of ACV beyond \lscv and address the issue of inexact optimization. In \cref{sec:structured_IJ}, we adapt the IJ framework to the \lwcv problem for structured models. In \cref{sec:inexact_optimization}, we show theoretically that both our new IJ approximation for \lwcv and the existing IJ approximation applied to \lscv are not overly dependent on having an exact optimum. We support both of these results with practical experiments in \cref{sec:exp}.

\section{Cross-validation and approximate cross-validation in structured models}
\label{sec:structured_IJ}

We first specify a weighting scheme, analogous to \cref{eq:old_approx_cv}, to describe \lwcv in structured models; then we develop an ACV method using this scheme.
Recall that CV in independent models takes various forms such as leave-$k$-out and $k$-fold CV.
Similarly, we consider the possibility of leaving
out\footnote{Note that the weight formulation could be extended to even more general reweightings in the spirit of the bootstrap.
Exploring the bootstrap for structured models is outside the scope of the present paper.}
multiple arbitrary sets of data indices $\leftoutinds \in \allfolds$, where each $\leftoutinds \subseteq [\numsteps]$.
We have two options for how to leave data out in hidden MRFs; see \cref{app:crf_cv} for CRFs. (A) For each data index $t$ left out, we leave out the data point $\xseqonet{t}$ but we retain the latent
$\zseqonet{t}$. For instance, in 
a time series, if data is missing in the middle of the series, we still know the time relation between the surrounding points, and would leave in the latent to maintain this relation.
(B) For each data index $t$ left out, we leave out the data point $\xseqonet{t}$ \emph{and} the latent $\zseqonet{t}$. For instance, consider data in the future of a time series or pixels beyond the edge of a picture. 
We typically would not include the possibility of all possible adjacent latents in such a structure, so leaving out $\zseqonet{t}$ as well is more natural. 
In either case, analogous to \cref{eq:old_approx_cv}, $\hat\param(\wseqone)$ is a function of $\wseqone$ computed by minimizing the negative log joint $-\log p(\xseqone; \Theta, \wseqone) - \log p(\Theta)$, now with $\wseqone$ dependence, in $\Theta$. For case (A), we adapt \cref{eq:mrf} (with $\numseq = 1$) and \cref{eq:fit_structure} with a weight $\wt$ for each $\xseqonet{t}$ term:
\begin{equation}
	\label{eq:mrf_cv_data_A}
	\hat\param(\wseqone)
		= \argmin{\param}
			Z(\param, \wseqone) + \int_{\zseqone}
		\left[
    			\sum_{t \in [T]} \wt \psi_{t}(\xseqonet{t}, \zseqonet{t}; \param)
    			\right]
    		+ \left[ 
    			\sum_{\textbf{c} \in \factors} \phi_{\textbf{c}}( \zseqonet{\textbf{c}}; \param)
    			\right]
		\; d\zseqone
		- \log p(\param)
		.
\end{equation}
Note that the negative log normalizing constant $Z(\param, \wseqone)$ may now depend on $\wseqone$ as well.
For case (B), we adapt \cref{eq:mrf} and \cref{eq:fit_structure} with a weight $\wt$ for each term with $\xseqonet{t}$ or $\zseqonet{t}$:
\begin{equation}
	\label{eq:mrf_cv_data_B}
	\hat\param(\wseqone)
		=  \argmin{\param}
			Z(\param, \wseqone) + \int_{\zseqone}
		\left[
    			\sum_{t \in [T]} \wt \psi_{t}(\xseqonet{t}, \zseqonet{t}; \param)
    			\right]
    		+ \left[ 
    			\sum_{\textbf{c} \in \factors} \left(\prod_{t \in \textbf{c}} \wt \right) \phi_{\textbf{c}}( \zseqonet{\textbf{c}}; \param)
    			\right]
		\; d\zseqone
		- \log p(\param)
		.
\end{equation}
In both cases, the choice $\wseqone = \mathbf{1}_{\numsteps}$ recovers the original learning problem. Likewise, setting $\wseqone = \wseqo$, where $\wseqo$ is a vector of ones with $w_t = 0$ if $t \in \leftoutinds$, drops out the data points in $\leftoutinds$ (and latents in case (B)). We show in \cref{app:chain_equivalence} that these two schemes are equivalent in the case of chain-structured graphs when $\leftoutinds = \{ T', T'+1, \dots, T\}$ but also that they are not equivalent in general. We thus consider both schemes going forward.

The expressions above allow a unifying viewpoint on \lwcv but still require re-solving $\hat\param(\wseqo)$ for each new CV fold $\leftoutinds$. To avoid this expense, we propose to use an IJ approach. In particular, as discussed by \citet{giordano2018swiss}, the intuition of the IJ is to notice that, subject to regularity conditions, a small change in $\wseqone$ induces a small change in $\hat\param(\wseqone)$. So we propose to approximate $\hat\param(\wseqo)$ with $\paramij(\wseqo)$, a first-order Taylor series expansion of $\hat\param(\wseqone)$ as a function of $\wseqone$ around $\wseqone = \mathbf{1}_{\numsteps}$. We follow \citet{giordano2018swiss} to derive this expansion in \cref{app:IJ_derivations}. Our IJ based approximation is applicable when the following conditions hold,
\begin{assumption}
 The model is fit via optimization (e.g.\ MAP or MLE).
\end{assumption}
\begin{assumption} 
	The model objective is twice differentiable and the Hessian matrix is invertible at the initial model fit $\hat\Theta$.
\end{assumption}
\begin{assumption} 
	The model fits across CV folds, $\paramlo$, can be written as optima of the same weighted objective for all folds $\mathbf{o}$ (e.g.\ as in \cref{eq:mrf_cv_data_A,eq:mrf_cv_data_B}).
\end{assumption}
We summarize our method and define $\paramalg$, with three arguments, in \cref{alg:main_algorithm}; we define $\paramij(\wseqo) := \paramalg(\paramone, \xseqone, \leftoutinds)$. 
\begin{algorithm}
\caption{Approximate leave-within-structure-out cross-validation for all folds $\leftoutinds \in \mathcal{O}$}
\label{alg:main_algorithm}
\begin{algorithmic}[1] 
\REQUIRE $\param_1, \xseqone, \mathcal{O}$\\
\STATE Define \emph{weighted} marginalization over $\zseqone$: $\log p(\xseqone; \param, \wseqone) =  \textsc{WeightedMarg}(\xseqone, \param, \wseqone)$. 
\STATE Compute $H = \frac{\partial^2 \log p(\xseqone; \param, \wseqone) + \log p(\param)}{\partial\param \partial \param^\top} \bigg|_{\param = \param_1, \wseqone = \mathbf{1}_\numsteps}$ 
\STATE Compute matrix $J = (J_{dt}) := \left( \frac{\partial^2 \log p(\xseqone; \param, \wseqone) + \log p(\param)}{\partial\param_d \partial \wt} \bigg|_{\param = \param_1, \wseqone = \mathbf{1}_\numsteps} \right)$
\STATE \textbf{for} $\leftoutinds \in \mathcal{O}$, \textbf{do:} \, $\paramalg(\param_1, \xseqone, \leftoutinds) := \displaystyle \param_1 + \sum_{t \in \leftoutinds}H^{-1}J_t$ \hfill \# $J_t$ is the $t$th column of $J$ \label{line:acv_for_loop}

\RETURN $\{ \paramalg(\param_{\mathbf{1}}, \xseqone, \leftoutinds) \}_{\leftoutinds \in \mathcal{O}}$
\end{algorithmic}
\end{algorithm}

First, note that the proposed procedure applies to either weighting style (A) or (B) above; they each determine a different $\log p(\xseqone, \param; \wseqone)$ in \cref{alg:main_algorithm}. We provide analogous \lscv algorithms for MRFs and CRFs in \cref{alg:supp_algorithm,alg:supp_algorithm2} (\cref{app:LSCV,app:crf_cv}).
Next, we compare the cost of our proposed ACV methods to exact CV. In what follows, we consider the initial learning problem $\paramoneseq$ a fixed cost and focus on runtime after that computation. We consider running CV for all folds $\leftoutinds \in \mathcal{O}$ in the typical case where the number of data points left out of each fold, $|\leftoutinds|$, is constant.
\begin{prop}
Let $M$ be the cost of a marginalization, i.e., running \textsc{WeightedMarg}; let $\numseq \ge 1$ be the number of independent structures; and let $S$ be the maximum number of steps used to fit the parameter in our optimization procedure. 
The cost of any one of our ACV algorithms (\cref{alg:main_algorithm,alg:supp_algorithm,alg:supp_algorithm2}) is in $O(M \numseq + D^3 + \numparams^2 \, |\leftoutinds| \, \numfolds)$. 
Exact CV is in $O(M \numseq S \numfolds)$.
\end{prop}
\begin{proof} For each of the $\numfolds$ folds of CV and each of the $\numseq$ structures, we compute the marginalization (cost $M$) at each of the $S$ steps of the optimization procedure. 
In our ACV algorithms, we compute $H$ and $J$ with automatic differentiation tools \citep{baydin2015automatic}. The results of \citet{biggs2000AD} demonstrate that $H$ and $J$ each require the same computation (up to a constant) as \textsc{WeightedMarg}. So, across $\numseq$, we incur cost $M \numseq$. 
We then incur a $O(D^3)$ cost to invert\footnote{In practice, for numerical stability, we compute a Cholesky factorization of $H$.} $H$.
The remaining cost is from the for loop.
\end{proof}

In structured problems, we generally expect $M$ to be large; see \cref{app:weighted_marg} for a discussion of the costs, including in the special case of chain-structured MRFs and CRFs. And for reliable CV, we want $\numfolds$ to be large. So we see that our ACV algorithms reap a savings by, roughly, breaking up the product of these terms into a sum and avoiding the further $S$ multiplier. 

\section{IJ behavior under inexact optimization}
\label{sec:inexact_optimization}

By envisioning the IJ as a Taylor series approximation around $\paramone$, the approximations for \lwcv (\cref{alg:main_algorithm}) and \lscv (\cref{alg:supp_algorithm,alg:supp_algorithm2} in the appendix) assume we have access to the exact optimum $\paramone$. In practice, though, especially in complex problems, computational considerations often require using an inexact optimum. More precisely, any optimization algorithm returns a sequence of parameter values $(\paramstep{s})_{s=1}^{S}$. Ideally the values $\paramstep{S}$ will approach the optimum $\hat\param(\mathbf{1}_{\numsteps})$ as $S \rightarrow \infty$. But we often choose $S$ such that $\paramstep{S}$ is much farther from $\hat\param(\mathbf{1}_{\numsteps})$ than machine precision. In practice, then, we input $\paramstep{S}$ (rather than $\hat\param(\mathbf{1}_{\numsteps})$) to \cref{alg:main_algorithm}. We now check that the error induced by this substitution is acceptably low.

We focus here on a particular use of CV: estimating out-of-sample loss. For simplicity, we discuss the $\numseq =1$ case here; see \cref{app:inexact_optimization} for the very similar $\numseq \geq 1$ case. 
For each fold $\leftoutinds \in \mathcal{O}$, we compute $\hat\param(\wseqo)$ from the points kept in and then calculate the loss (in our experiments here, negative log likelihood) on the left-out points. I.e.\ the CV estimate of the out-of-sample loss is
$
	\losscv := (1/\numfolds) \sum_{\mathbf{o} \in \mathcal{O}}
			-\log p(\xseqonet{\mathbf{o}} \mid \xseqonet{[\numsteps]\mathbf{-o}}; \hat\param(\wseqo)),
$
where $-\log p$ may come from either weighting scheme (A) or (B). 
See \cref{app:inexact_optimization} for an extension to CV computed with a generic loss $\ell$.
We approximate $\losscv$ using some $\param$ as input to \cref{alg:main_algorithm}; we denote this approximation by
$
	\lossij(\param) := (1/\abs{\mathcal{O}}) \sum_{\leftoutinds \in \mathcal{O}} -\log p(x_{\leftoutinds} \mid x_{[\numsteps] - \leftoutinds} ; \paramalg(\param, \xseqone, \leftoutinds )).
$

Below, we will bound the error in our approximation: $|\losscv - \lossij(\paramstep{S})|$. There are two sources of error. (1) The difference in loss between exact CV and the exact IJ approximation, $\epsij$ in \cref{eq:inexact_optimization_constants}. (2) The difference in the parameter value, $\epsparam$ in \cref{eq:inexact_optimization_constants}, which will control the difference between $\lossij(\paramone)$ and $\lossij(\paramstep{S})$.
\begin{equation} \label{eq:inexact_optimization_constants}
	\epsij := |\losscv - \lossij(\paramone)|,
	\quad
	\epsparam := \| \paramstep{S} - \hat\param(\mathbf{1}_{\numsteps}) \|_2
\end{equation}
Our bound below will depend on these constants. We observe that empirics, as well as theory based on the Taylor series expansion underlying the IJ, have established that $\epsij$ is small in various models; we expect the same to hold here. Also, $\epsparam$ should be small for large enough $S$ according to the guarantees of standard optimization algorithms. We now state some additional regularity assumptions before our main result.

\begin{assumption}
\label{assum:inexact_optimization}
  Take any ball $B \subset \R^D$ centered on $\paramone$ and containing $\paramstep{S}$. We assume the objective $-\log p(\xseqone ; \param, \mathbf{1}_\numsteps) - p(\param)$ is strongly convex with parameter $\lambda_{\mathrm{min}}$ on $B$. Additionally, on $B$, we assume the derivatives $g_t(\param) := \partial^2 \log p(\xseqone ; \param, \wseqone) / \partial \param \partial w_t$ are Lipschitz continuous with constant $L_g$ for all $t$, and the inverse Hessian of the objective is Lipschitz with parameter $L_{Hinv}$. Finally, on $B$, take $\log p(\xseqone; \param, \wseqo)$ to be a Lipschitz function of $\param$ with parameter $L_p$ for all $\wseqo$.
\end{assumption}
We make a few remarks on the restrictiveness of these assumptions. First, while few structured models have objectives that are even convex (e.g., the label switching problem for HMMs guarantees non-convexity), we expect most objectives to be locally convex around an exact minimum $\paramone$; \cref{assum:inexact_optimization} requires that the objective in fact be strongly locally convex. Next, while the Lipschitz assumption on the $g_t$ may be hard to interpret in general, we note that it takes on a particularly simple form in the setup of \cref{eq:old_approx_cv}, where we have $g_t = \nabla f_t$. Finally, we note that the condition that the inverse Hessian is Lipschitz is not much of an additional restriction. E.g., if $\nabla p(\param)$ is also Lipschitz continuous, then the entire objective has a Lipschitz gradient, and so its Hessian is bounded. As it is also bounded below by strong convexity, we find that the inverse Hessian is bounded above and below, and thus is Lipschitz continuous. We now state our main result.
\begin{prop}
\label{prop:inexact_optimization}
   The approximation error of $\lossij(\paramstep{S})$ satisfies the following bound:
  \begin{equation} 
  	\abs{ \lossij(\paramstep{S}) - \losscv } \leq C \eps_{\theta} + \eps_{\mathrm{IJ}}, 
  \end{equation}
   $$
   	\textrm{ where } \quad C := 
	L_p + \frac{L_p L_g}{\lambda_{\mathrm{min}}} + \frac{L_p L_{Hinv}}{\numfolds} \sum_{\mathbf{o} \in \mathcal{O}}   \n{ \sum_{t \in \mathbf{o}} \nabla g_t(\paramone) }_2. 	
$$
\end{prop}
See \cref{app:inexact_optimization} for a proof. Note that, while $C$ may depend on $T$ or $\mathcal{O}$, we expect it to approach a constant as $T \to\infty$ under mild distributional assumptions on $\| g_t \|_2$; see \cref{app:inexact_optimization}.
We finally note that although the results of this section are motivated by structured models, they apply to, and are novel for, the simpler models considered in previous work on ACV methods.
\section{Experiments}
\label{sec:exp}

We demonstrate the effectiveness of our proposed ACV methods on a diverse set of real-world examples where data exhibit temporal and spatial dependence: namely, temporal count modeling, named entity recognition, and spatial modeling of crime data. Additional experiments validating the accuracy and computational benefits afforded by \lscv are available in \cref{subsec:mocap}, where we explore auto-regressive HMMs for motion capture analysis -- with $\numseq = 124$, $\numsteps$ up to 100, and $\numparams$ up to 11{,}712. 


\textbf{Approximate leave-within-sequence-out CV: Time-varying Poisson processes.}
We begin by examining approximate \lwcv (Algorithm~\ref{alg:main_algorithm}) for maximum a posteriori (MAP) estimation. We consider a time-varying Poisson process model used by \citep{ihler2006adaptive} for detecting events in temporal count data. We analyze loop sensor data collected every five minutes over a span of 25 weeks from a section of a freeway near a baseball stadium in Los Angeles. For this problem, there is one observed sequence ($\numseq = 1$) with $\numsteps = 50{,}400$ total observations. There are $\numparams = 11$ parameters. Full model details are in \cref{subsec:dodger}.

To choose the folds in both exact CV and our ACV method, we consider two schemes, both following style (A) in \cref{eq:mrf_cv_data_A}; i.e., we omit observations (but not latents) in the folds. First, we follow the recommendation of \citet{celeux2008selecting}; namely, we form each fold by selecting $m\%$ of measurements to omit (i.e., to form $\leftoutinds$) uniformly at random and independently across folds. We call this scheme \emph{i.i.d.\ \lwcv}. Second, we consider a variant where we omit $m\%$ of observations in a contiguous block. We call this scheme \emph{contiguous \lwcv}; see \cref{subsec:dodger}.

In evaluating the accuracy of our approximation, we focus on a subset of $\numsteps_{sub} = 10{,}000$ observations, plotted in the top panel of \cref{fig:dodgers_withinseq}. The six panels in the lower left of \cref{fig:dodgers_withinseq} compare our ACV estimates to exact CV. Columns range over left-out percentages $m=2,5,10$ (all on the data subset); rows depict i.i.d. \lwcv (upper) and contiguous CV (lower). For each of $\numfolds = 10$ folds and for each point $\xseqonet{t}$ left out in each fold, we plot a red dot with the exact fold loss $-\log p(\xseqonet{t} \mid \xobs; \hat\param(\wseq{\leftoutinds}))$ as its horizontal coordinate and our approximation $-\log p(\xseqonet{t} \mid \xobs; \paramij(\wseq{\leftoutinds}))$ as its vertical coordinate. We can see that every point lies close to the dashed black $x=y$ line; that is, the quality of our approximation is uniformly high across the thousands of points in each plot.

In the two lower right panels of \cref{fig:dodgers_withinseq}, we compare the speed of exact CV to our approximation and the Newton step (NS) approximation~\citep{beirami2017firstALOO, rad2018detailedALOO} on two data subsets (size 5,000 and 10,000) and the full data. No reported times include the initial $\paramone$ computation since $\paramone$ represents the unavoidable cost of the data analysis itself. I.i.d.\ \lwcv appears in the upper plot, and contiguous \lwcv appears in the lower. For our approximation, we use 1,000 folds. Due to the prohibitive cost of both exact CV and NS, we run them for 10 folds and multiply by 100 to estimate runtime over 1,000 folds. We see that our approximation confers orders of magnitude in time savings both over exact CV and approximations based on NS. 
In \cref{app:IJ_v_NS}, we show that the approximations based on NS do not substantatively improve upon those provided by the significantly cheaper IJ approximations.

\begin{figure}[ht]
\centering
\includegraphics[width=1\textwidth]{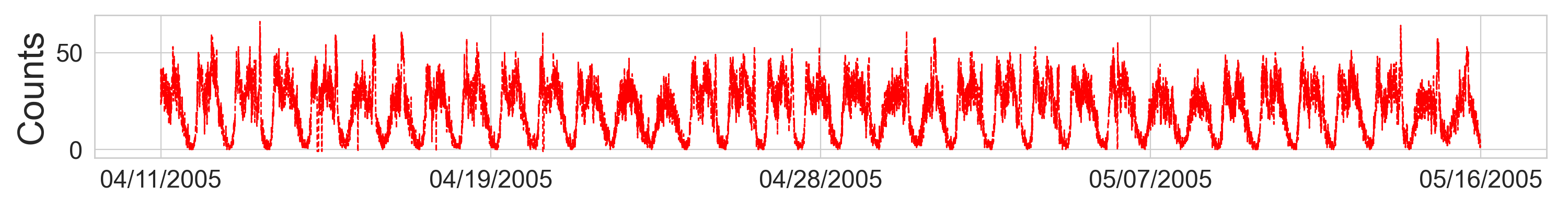}\\
\includegraphics[width=1\textwidth]{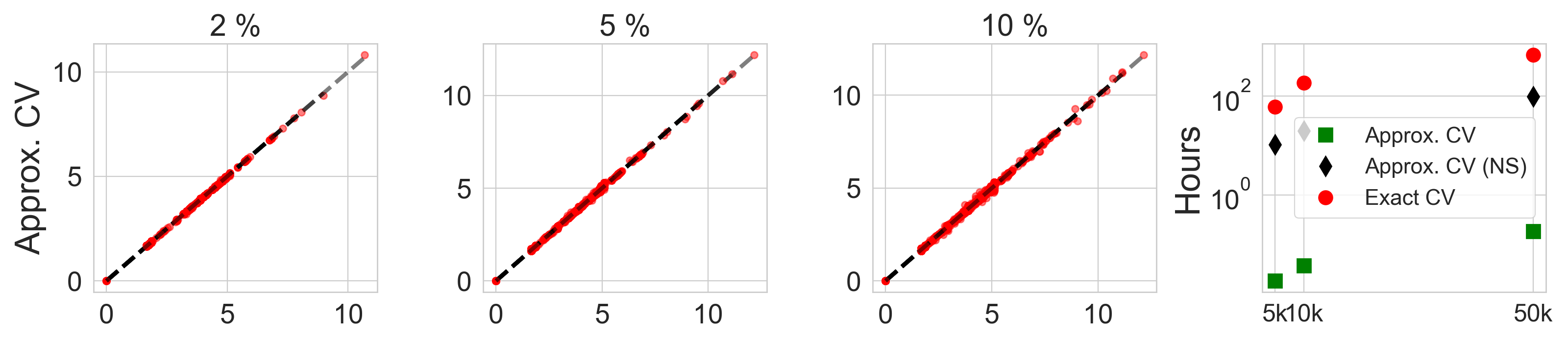}
\includegraphics[width=1\textwidth]{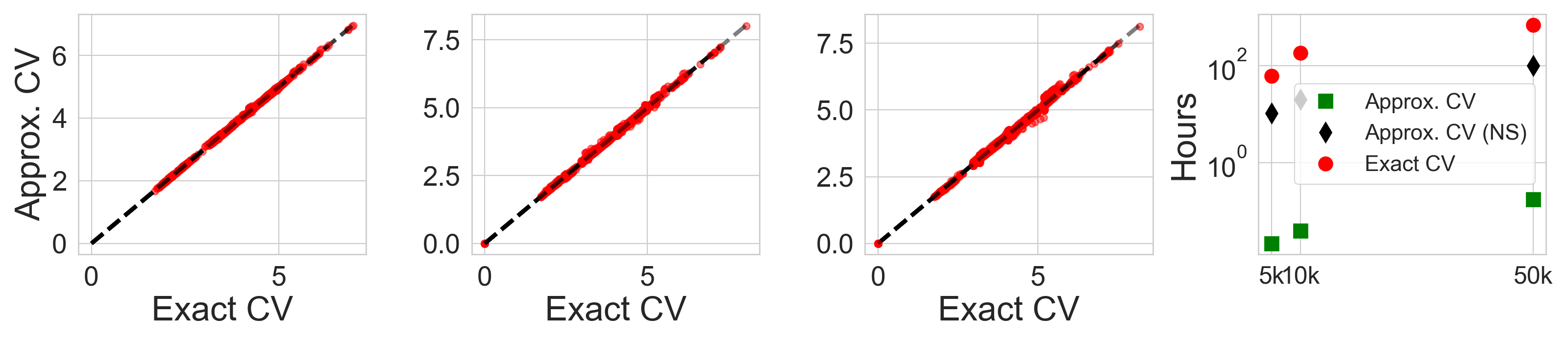}
	\vskip -10pt
\caption{Evaluation of approximate \lwcv for time-varying Poisson processes. \emph{(Top panel)} A subset of the count series. \emph{(Lower left six panels)} Scatter plots comparing exact CV loss (horizontal axis) at each point in each fold (red dots) to our approximation of CV loss (vertical axis). Black dashed line shows perfect agreement. Three columns for percent points left out; two rows for i.i.d.\ \lwcv (upper) and contiguous \lwcv (lower). \emph{(Lower right two panels)} Wall-clock time for exact and approximate CV measured on a 2.5GHz quad core Intel i7 processor with 16GB of RAM; same rows as left panels.}
\label{fig:dodgers_withinseq}
\end{figure}

\textbf{Robustness to inexact optimization: Neural conditional random fields.}
\begin{figure}[ht]
\centering
\includegraphics[width=1\textwidth]{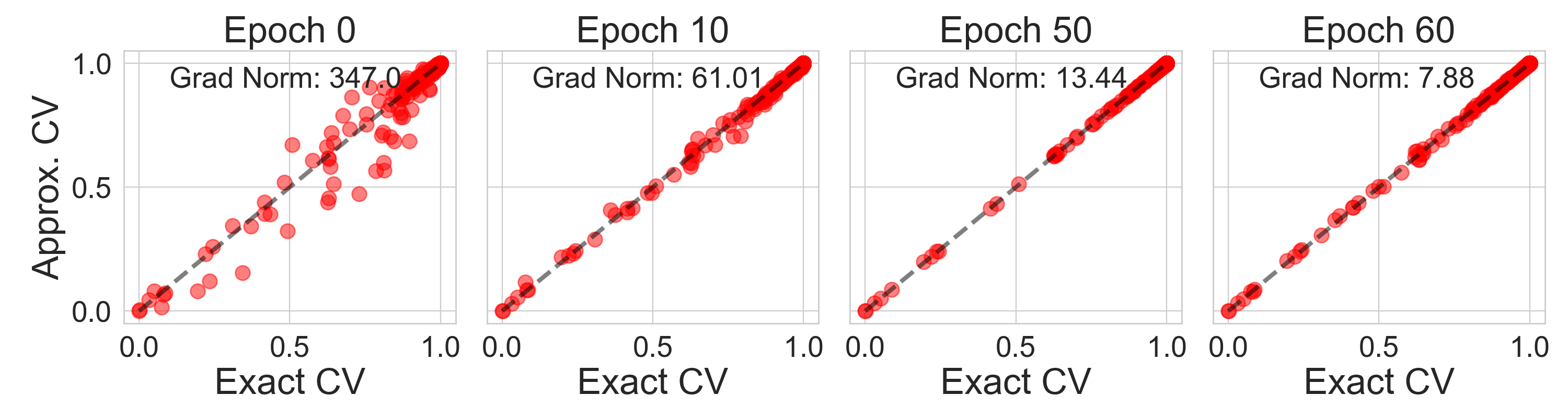}
\vskip -5pt
\caption{Behavior of ACV at different epochs in stochastic optimization for a bilstm-crf. Scatter plots comparing held out probabilities under CV  (horizontal axis) at each point in each fold (red dots) to our approximation of CV (vertical axis). Black dashed line shows perfect agreement. 
}
\label{fig:crf}
\end{figure}
Next, we examine the effect of using an inexact optimum $\paramstep{S}$, instead of the exact initial optimum $\paramone$, as the input in our approximations.
We consider \lscv for a bidirectional LSTM CRF (bilstm-crf)~\citep{huang2015bidirectional}, which has been found~\citep{lample2016neural, ma2016end, reimers2017optimal} to perform well for named entity recognition. 
In this case, our problem is supervised; the words in input sentences ($\xseq$) are annotated with entity labels ($\zseq$), such as organizations or locations. We trained the bilstm-crf model on the CoNLL-2003 shared task benchmark~\citep{sang2003introduction} using the English subset of the data and the pre-defined train/validation/test splits containing 14,987(=$\numseq$)/3,466/3,684 sentence annotation pairs. Here $\numsteps$ is the number of words in a sentence; it varies by sentence with a max of 113 and median of 9. The number of parameters $\numparams$ is 99. Following standard practice, we optimize the full model using stochastic gradient methods and employ early stopping by monitoring loss on the validation set. See \cref{subsec:neural_crf} for model architecture and optimization details. In our experiments, we hold the other network layers (except for the CRF layer) fixed, and report epochs for training on the CRF layer after full-model training; this procedure mimics some transfer learning methods \citep{huh2016makes}.

We consider 500 \lscv folds with one sentence (i.e., one $n$ index) per fold; the 500 points are chosen uniformly at random.
The four panels in \cref{fig:crf} show the behavior of our approximation (\cref{alg:supp_algorithm2} in \cref{app:crf_cv}) at different training epochs during the optimization procedure. 
To ensure invertibility of the Hessian when far from an optimum, we add a small ($10^{-5}$) regularizer to the diagonal. At each epoch, for each fold, we plot a red dot with the exact fold held out probability $ p(\zseq \mid \xseq ;\hat\param(\wseq{\{n\}})$ as its horizontal coordinate and our approximation $p(\zseq \mid \xseq ; \paramij(\wseq{\{n\}})$ as the vertical coordinate. Note that the \lscv loss has no dependence on other $n$ due to the model independence across $n$; see \cref{app:crf_cv}. 
Even in early epochs with larger gradient norms, every point lies close to the dashed black $x=y$ line. \cref{supp:crf_v_time} of \cref{subsec:neural_crf} further shows the mean absolute approximation error between the exact CV held out probability and our approximation, across all 500 folds as a function of log gradient norm and wall clock time. As expected, our approximation has higher quality at better initial fits. Nonetheless, we see that decay in performance away from the exact optimum is gradual.

\textbf{Beyond chain-structured graphs: Crime statistics in Philadelphia.} 
The models in our experiments above are all chain-structured. Next we consider our approximations to \lwcv in a spatial model with more complex dependencies. \citet{BalocchiJensen2019, Balocchi2020} have recently studied spatial models of crime in the city of Philadelphia. We here consider a (simpler) hidden MRF model for exposition: a Poisson mixture with spatial dependencies, detailed in \cref{subsec:philly}. Here, there is a single structure observation ($\numseq = 1$); there are $\numsteps = 384$ census tracts in the city; and there are $\numparams = 2$ parameters. The data is shown in the upper lefthand panel of \cref{fig:crime_tracts}.

We choose one point per fold in style (B) of \lwcv here, for a total of 384 folds.
We test our method across four fixed values of a hyperparameter $\beta$ that encourages adjacent tracts to be in the same latent state.
For each fold, we plot a red dot comparing the exact fold loss $-\log p(\xseqonet{t} \mid \mathbf{x}_{[\numsteps]-\{t\}}; \hat\param(\wseq{\{t\}}))$ with our approximation $-\log p(\xseqonet{t} \mid \mathbf{x}_{[\numsteps]-\{t\}}; \paramij(\wseq{\{t\}}))$. The results are in the lower four panels of \cref{fig:crime_tracts}, where we see uniformly small error across folds in our approximation. In the upper right panel of \cref{fig:crime_tracts}, we see that our method is orders of magnitude faster than exact CV.

\begin{figure}[ht]
\centering
	\includegraphics[width = 0.4\textwidth]{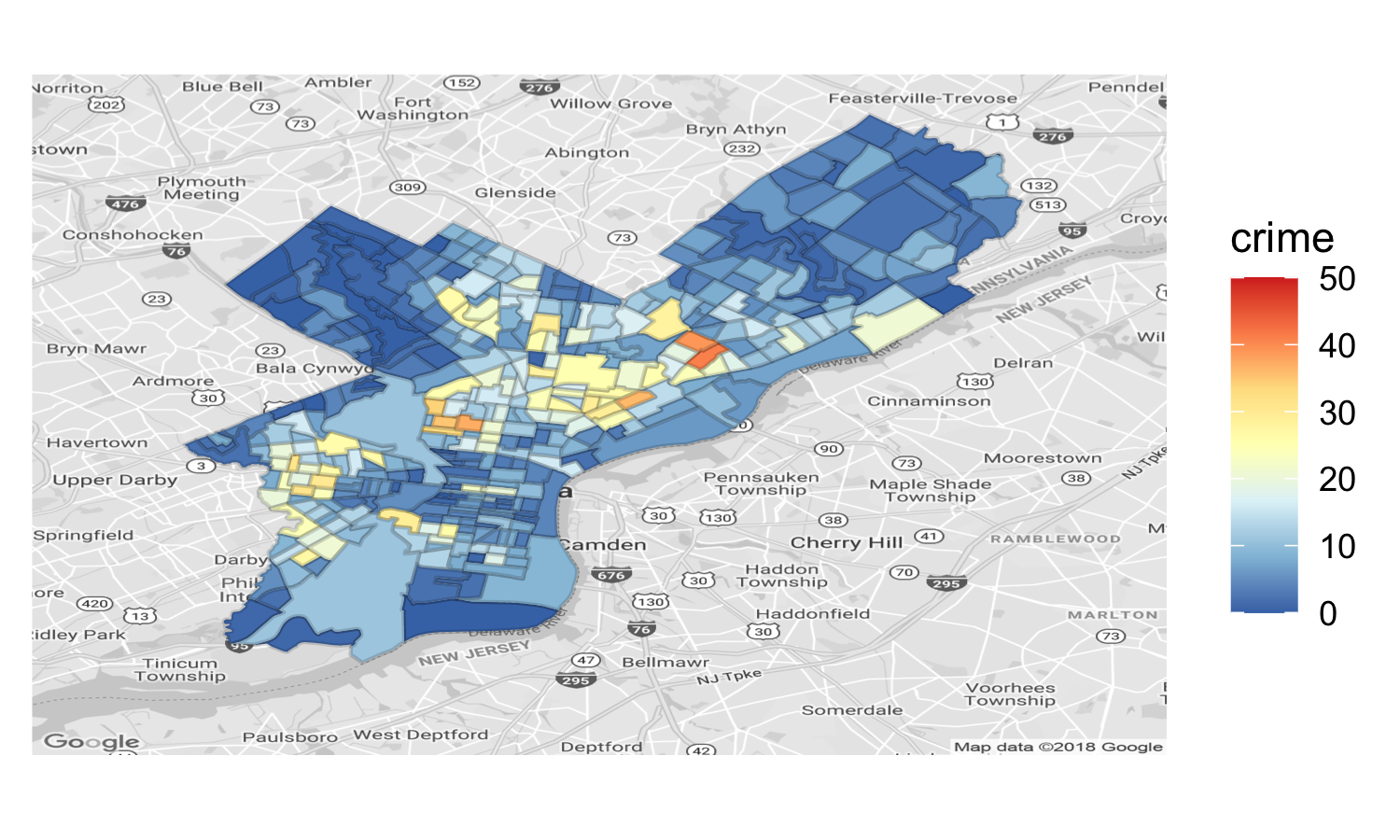}
	\includegraphics[width = 0.45\textwidth]{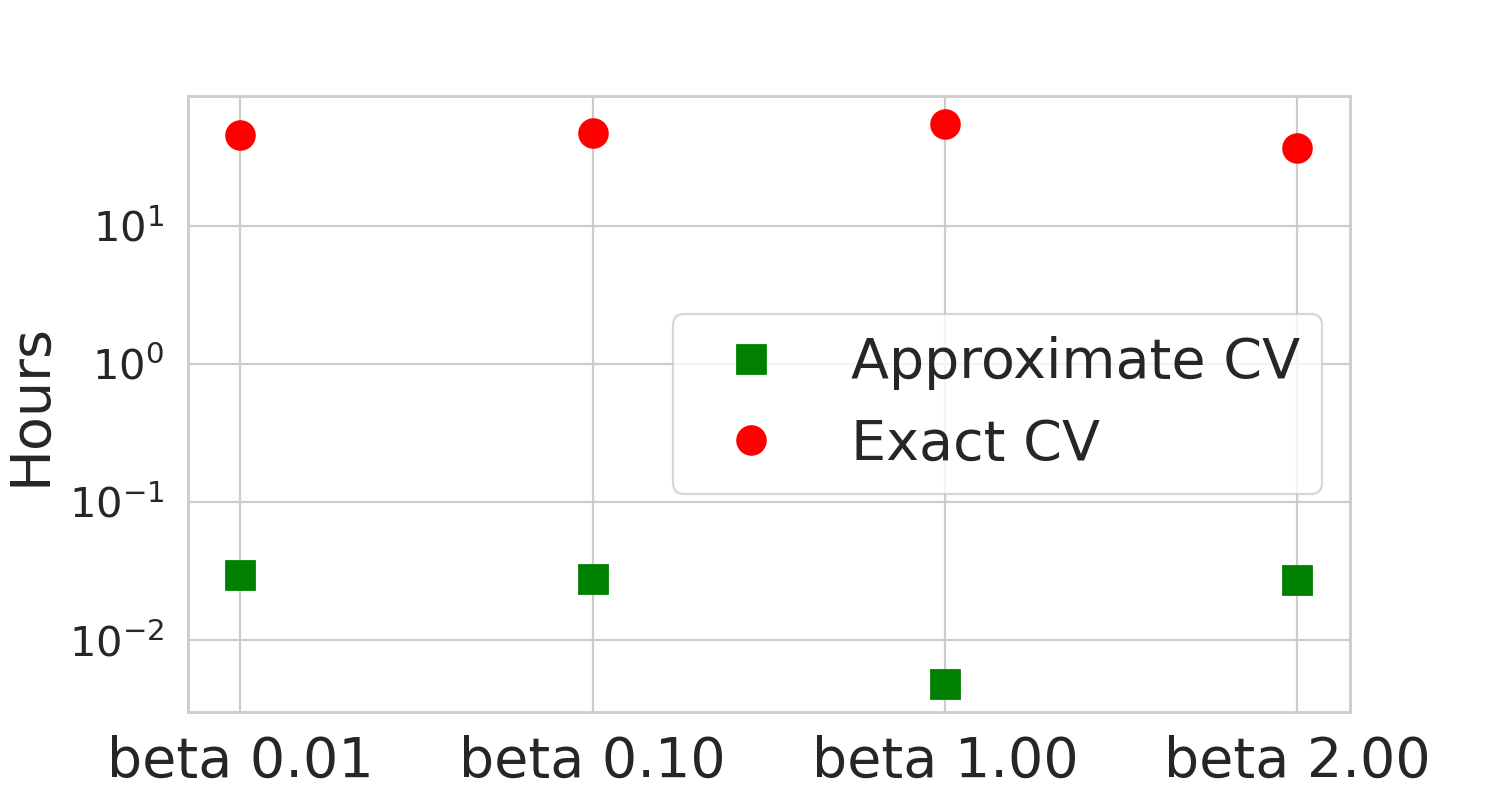}
	\includegraphics[width = 0.95\textwidth]{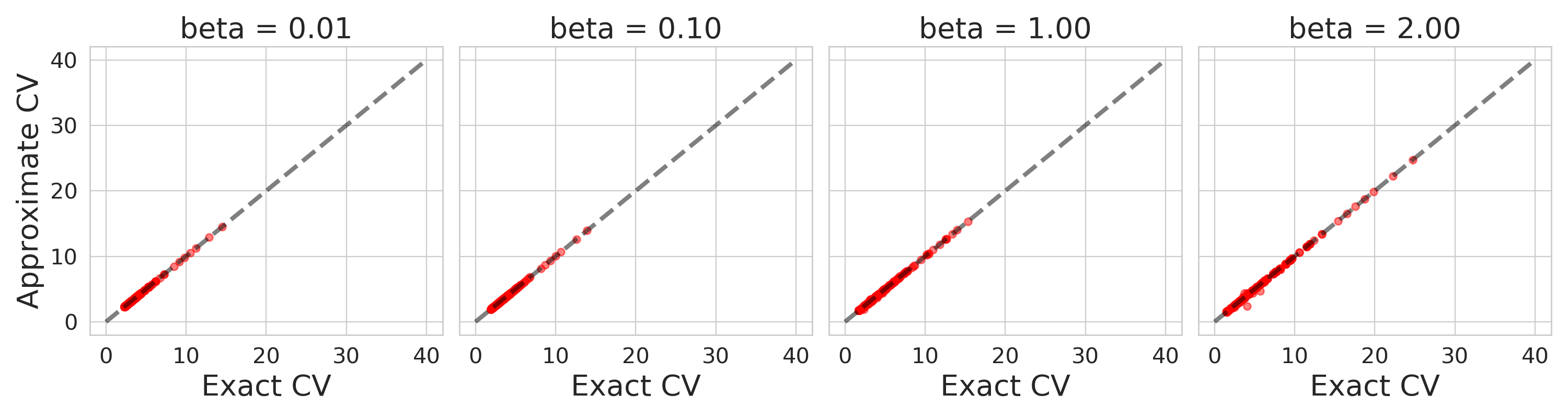}
	\vskip -5pt
	\caption{Evaluation of \lwcv for loopy Markov random field. \emph{(Top left)} Census tracts data. \emph{(Upper right)} Wall-clock time of approximate CV and exact CV. \emph{(Lower)} Scatter plots comparing CV loss (horizontal axis) at each point in each fold (red dots) to our approximation of CV loss (vertical axis). Black dashed line shows perfect agreement. Plots generated with different values of connectivity $\beta$.}
	\label{fig:crime_tracts}
\end{figure}
\vskip -15pt

\paragraph{Discussion.}

In this work, we have demonstrated how to extend approximate cross-validation (ACV) techniques to CV tasks with non-trivial dependencies between folds. We have also demonstrated that IJ approximations can retain their usefulness even when the initial data fit is inexact. While our motivation in the latter case was formed by complex models of dependent structures, our results are also applicable to, and novel for, the classic independence
framework of ACV. An interesting remaining challenge for future work is to address other sources of computational expense in structured models. For instance, even after computing $\paramlo$, inference can be expensive in very large graphical models; it remains to be seen if reliable and fast approximations can be found for this operation as well.


\section*{Broader Impact}
Accurate evaluation enables more reliable machine learning methods
and more trustworthy communication of their capabilities. To the
extent that machine learning methods may be  beneficial -- in that
they may be used to facilitate medical diagnosis, assistive technology
for individuals with motor impairments, or understanding of helpful
economic interventions -- accurate evaluation ensures these benefits
are fully realized. To the extent that machine learning methods may be
harmful -- in that they may used to facilitate the spread of false
information or privacy erosion -- accurate evaluation should still
make these methods more effective at their goals, even if societally
undesirable. As in any machine learning methodology, it is also important for the buyer to beware; 
while we have tried to pick a broad array of experimental settings and to support our methods with theory, 
there may remain cases of interest when our approximations fail without warning. In fact,
we take care to note that cross-validation and its points of failure are still not fully understood.
All of our results are relative to exact cross-validation -- since it is taken as the de facto standard
for evaluation in the machine learning community (not without reason \citep{musgrave2020metric}). But when
exact cross-validation fails, we therefore expect our method to fail as well.

	\paragraph{Acknowledgments.} 
	This work was supported by the MIT-IBM Watson AI Lab, DARPA, the CSAIL--MSR Trustworthy AI Initiative, an NSF CAREER Award, an ARO YIP Award, ONR, and Amazon. Broderick Group is also supported by the Sloan Foundation, ARPA-E, Department of the Air Force, and MIT Lincoln Laboratory. 

\bibliographystyle{abbrvnat}
\bibliography{refs}

\begin{thebibliography}{46}
\providecommand{\natexlab}[1]{#1}
\providecommand{\url}[1]{\texttt{#1}}
\expandafter\ifx\csname urlstyle\endcsname\relax
  \providecommand{\doi}[1]{doi: #1}\else
  \providecommand{\doi}{doi: \begingroup \urlstyle{rm}\Url}\fi

\bibitem[Balocchi and Jensen(2019)]{BalocchiJensen2019}
C.~Balocchi and S.~T. Jensen.
\newblock Spatial modeling of trends in crime over time in {P}hiladelphia.
\newblock \emph{The Annals of Applied Statistics}, 13\penalty0 (4):\penalty0
  2235--2259, 2019.

\bibitem[Balocchi et~al.(2019)Balocchi, Deshpande, George, and
  Jensen]{Balocchi2020}
C.~Balocchi, S.~K. Deshpande, E.~I. George, and S.~T. Jensen.
\newblock Crime in {P}hiladelphia: {B}ayesian clustering with particle
  optimization.
\newblock \emph{arXiv preprint arXiv:1912.00111}, 2019.

\bibitem[Bartholomew-Biggs et~al.(2000)Bartholomew-Biggs, Brown, Christianson,
  and Dixon]{biggs2000AD}
M.~Bartholomew-Biggs, S.~Brown, B.~Christianson, and L.~Dixon.
\newblock Automatic differentiation of algorithms.
\newblock \emph{Journal of Computational and Applied Mathematics}, 124, 2000.

\bibitem[Baydin et~al.(2018)Baydin, Pearlmutter, Radul, and
  Siskind]{baydin2015automatic}
A.~G. Baydin, B.~A. Pearlmutter, A.~A. Radul, and J.~M. Siskind.
\newblock Automatic differentiation in machine learning: {A} survey.
\newblock \emph{arXiv Preprint arXiv:1502.05767v4}, 2018.

\bibitem[Beirami et~al.(2017)Beirami, Razaviyayn, Shahin, and
  Tarokh]{beirami2017firstALOO}
A.~Beirami, M.~Razaviyayn, S.~Shahin, and V.~Tarokh.
\newblock On optimal generalizability in parametric learning.
\newblock In \emph{Advances in Neural Information Processing Systems (NIPS)},
  pages 3458--3468, 2017.

\bibitem[Bellare and McCallum(2007)]{bellare2007learning}
K.~Bellare and A.~McCallum.
\newblock Learning extractors from unlabeled text using relevant databases.
\newblock In \emph{Sixth {I}nternational {W}orkshop on {I}nformation
  {I}ntegration on the {W}eb}, 2007.

\bibitem[B{\"u}rkner et~al.(2020)B{\"u}rkner, Gabry, and
  Vehtari]{burkner2019approximate}
P.-C. B{\"u}rkner, J.~Gabry, and A.~Vehtari.
\newblock Approximate leave-future-out cross-validation for {B}ayesian time
  series models.
\newblock \emph{arXiv preprint arXiv:1902.06281}, may 2020.

\bibitem[Celeux and Durand(2008)]{celeux2008selecting}
G.~Celeux and J.-B. Durand.
\newblock Selecting hidden {M}arkov model state number with cross-validated
  likelihood.
\newblock \emph{Computational Statistics}, 23\penalty0 (4):\penalty0 541--564,
  2008.

\bibitem[DeCaprio et~al.(2007)DeCaprio, Vinson, Pearson, Montgomery, Doherty,
  and Galagan]{decaprio2007conrad}
D.~DeCaprio, J.~P. Vinson, M.~D. Pearson, P.~Montgomery, M.~Doherty, and J.~E.
  Galagan.
\newblock Conrad: {G}ene prediction using conditional random fields.
\newblock \emph{Genome research}, 17\penalty0 (9):\penalty0 1389--1398, 2007.

\bibitem[Dempster et~al.(1977)Dempster, Laird, and Rubin]{dempster1977maximum}
A.~P. Dempster, N.~M. Laird, and D.~B. Rubin.
\newblock Maximum likelihood from incomplete data via the {EM} algorithm.
\newblock \emph{Journal of the Royal Statistical Society: Series B
  (Methodological)}, 39\penalty0 (1):\penalty0 1--22, 1977.

\bibitem[Efron(1981)]{efron1981nonparametric}
B.~Efron.
\newblock Nonparametric estimates of standard error: {T}he jackknife, the
  bootstrap and other methods.
\newblock \emph{Biometrika}, 68\penalty0 (3):\penalty0 589--599, 1981.

\bibitem[Fox et~al.(2009)Fox, Jordan, Sudderth, and Willsky]{fox2009}
E.~Fox, M.~I. Jordan, E.~B. Sudderth, and A.~S. Willsky.
\newblock Sharing features among dynamical systems with {B}eta processes.
\newblock In \emph{Advances in Neural Information Processing Systems (NIPS)},
  pages 549--557, 2009.

\bibitem[Fox et~al.(2014)Fox, Hughes, Sudderth, and Jordan]{fox2014}
E.~B. Fox, M.~C. Hughes, E.~B. Sudderth, and M.~I. Jordan.
\newblock Joint modeling of multiple time series via the {B}eta process with
  application to motion capture segmentation.
\newblock \emph{The Annals of Applied Statistics}, 8\penalty0 (3):\penalty0
  1281--1313, 2014.

\bibitem[Geisser(1975)]{geisser1975earlyCV}
S.~Geisser.
\newblock The predictive sample reuse method with applications.
\newblock \emph{Journal of the American Statistical Association}, 70\penalty0
  (350):\penalty0 320--328, June 1975.

\bibitem[Giordano et~al.(2019)Giordano, Stephenson, Liu, Jordan, and
  Broderick]{giordano2018swiss}
R.~Giordano, W.~T. Stephenson, R.~Liu, M.~I. Jordan, and T.~Broderick.
\newblock A {S}wiss {A}rmy infinitesimal jackknife.
\newblock In \emph{International Conference on Artificial Intelligence and
  Statistics (AISTATS)}, 2019.

\bibitem[Huang et~al.(2015)Huang, Xu, and Yu]{huang2015bidirectional}
Z.~Huang, W.~Xu, and K.~Yu.
\newblock Bidirectional {LSTM-CRF} models for sequence tagging.
\newblock \emph{arXiv preprint arXiv:1508.01991}, 2015.

\bibitem[Hughes and Sudderth(2014)]{bnpy}
M.~C. Hughes and E.~B. Sudderth.
\newblock Bnpy: {R}eliable and scalable variational inference for {B}ayesian
  nonparametric models.
\newblock In \emph{NIPS Probabilistic Programimming Workshop}, pages 8--13,
  2014.

\bibitem[Hughes et~al.(2012)Hughes, Fox, and Sudderth]{hughes2012}
M.~C. Hughes, E.~Fox, and E.~B. Sudderth.
\newblock Effective split-merge {M}onte {C}arlo methods for nonparametric
  models of sequential data.
\newblock In \emph{Advances in Neural Information Processing Systems (NIPS)},
  pages 1295--1303, 2012.

\bibitem[Huh et~al.(2016)Huh, Agrawal, and Efros]{huh2016makes}
M.~Huh, P.~Agrawal, and A.~A. Efros.
\newblock What makes imagenet good for transfer learning?
\newblock \emph{arXiv preprint arXiv:1608.08614}, 2016.

\bibitem[Hyndman and Athanasopoulos(2018)]{hyndman2018forecasting}
R.~J. Hyndman and G.~Athanasopoulos.
\newblock \emph{Forecasting: {P}rinciples and practice}.
\newblock OTexts: Melbourne, Australia, 2018.

\bibitem[Ihler et~al.(2006)Ihler, Hutchins, and Smyth]{ihler2006adaptive}
A.~Ihler, J.~Hutchins, and P.~Smyth.
\newblock Adaptive event detection with time-varying {P}oisson processes.
\newblock In \emph{ACM SIGKDD International Conference on Knowledge Discovery
  and Data Mining}, pages 207--216, 2006.

\bibitem[Jaeckel(1972)]{jaeckel1972infinitesimal}
L.~Jaeckel.
\newblock The infinitesimal jackknife, memorandum.
\newblock Technical report, MM 72-1215-11, Bell Lab. Murray Hill, NJ, 1972.

\bibitem[Koh and Liang(2017)]{koh2017understanding}
P.~W. Koh and P.~Liang.
\newblock Understanding black-box predictions via influence functions.
\newblock In \emph{International Conference on Machine Learning (ICML)}, pages
  1885--1894. JMLR. org, 2017.

\bibitem[Koh et~al.(2019)Koh, Ang, Teo, and Liang]{koh2019accuracy}
P.~W.~W. Koh, K.-S. Ang, H.~Teo, and P.~S. Liang.
\newblock On the accuracy of influence functions for measuring group effects.
\newblock In \emph{Advances in Neural Information Processing Systems}, pages
  5255--5265, 2019.

\bibitem[Koller and Friedman(2009)]{koller2009probabilistic}
D.~Koller and N.~Friedman.
\newblock \emph{Probabilistic graphical models: {P}rinciples and techniques -
  adaptive computation and machine learning}.
\newblock The MIT Press, 2009.
\newblock ISBN 0262013193.

\bibitem[Lample et~al.(2016)Lample, Ballesteros, Subramanian, Kawakami, and
  Dyer]{lample2016neural}
G.~Lample, M.~Ballesteros, S.~Subramanian, K.~Kawakami, and C.~Dyer.
\newblock Neural architectures for named entity recognition.
\newblock \emph{arXiv preprint arXiv:1603.01360}, 2016.

\bibitem[Ma and Hovy(2016)]{ma2016end}
X.~Ma and E.~Hovy.
\newblock End-to-end sequence labeling via bi-directional {LSTM-CNNs-CRF}.
\newblock \emph{arXiv preprint arXiv:1603.01354}, 2016.

\bibitem[Musgrave et~al.(2020)Musgrave, Belongie, and Lim]{musgrave2020metric}
K.~Musgrave, S.~Belongie, and S.-N. Lim.
\newblock A metric learning reality check.
\newblock \emph{arXiv preprint arXiv:2003.08505}, 2020.

\bibitem[Obuchi and Kabashima(2016)]{obuchi2016linearALOO}
T.~Obuchi and Y.~Kabashima.
\newblock Cross validation in {LASSO} and its acceleration.
\newblock \emph{Journal of Statistical Mechanics}, May 2016.

\bibitem[Obuchi and Kabashima(2018)]{obuchi2018logisticALOO}
T.~Obuchi and Y.~Kabashima.
\newblock Accelerating cross-validation in multinomial logistic regression with
  l1-regularization.
\newblock \emph{Journal of Machine Learning Research}, Sept. 2018.

\bibitem[Pennington et~al.(2014)Pennington, Socher, and
  Manning]{pennington2014glove}
J.~Pennington, R.~Socher, and C.~D. Manning.
\newblock {G}love: {G}lobal vectors for word representation.
\newblock In \emph{Conference on Empirical Methods in Natural Language
  Processing (EMNLP)}, pages 1532--1543, 2014.

\bibitem[Rad and Maleki(2020)]{rad2018detailedALOO}
K.~R. Rad and A.~Maleki.
\newblock {A scalable estimate of the extra-sample prediction error via
  approximate leave-one-out}.
\newblock \emph{arXiv Preprint arXiv:1801.10243v4}, Jan. 2020.

\bibitem[Rangel et~al.(2004)Rangel, Angus, Ghahramani, Lioumi, Sotheran, Gaiba,
  Wild, and Falciani]{rangel2004modeling}
C.~Rangel, J.~Angus, Z.~Ghahramani, M.~Lioumi, E.~Sotheran, A.~Gaiba, D.~L.
  Wild, and F.~Falciani.
\newblock Modeling {T}-cell activation using gene expression profiling and
  state-space models.
\newblock \emph{Bioinformatics}, 20\penalty0 (9):\penalty0 1361--1372, 2004.

\bibitem[Reimers and Gurevych(2017)]{reimers2017optimal}
N.~Reimers and I.~Gurevych.
\newblock Optimal hyperparameters for deep {LSTM}-networks for sequence
  labeling tasks.
\newblock \emph{arXiv preprint arXiv:1707.06799v2}, 2017.

\bibitem[Sang and De~Meulder(2003)]{sang2003introduction}
E.~T.~K. Sang and F.~De~Meulder.
\newblock {Introduction to the CoNLL-2003 Shared Task: Language-Independent
  Named Entity Recognition}.
\newblock In \emph{Conference on Natural Language Learning at HLT-NAACL 2003},
  pages 142--147, 2003.

\bibitem[Stephenson and Broderick(2020)]{stephenson2019sparse}
W.~T. Stephenson and T.~Broderick.
\newblock Approximate cross-validation in high dimensions with guarantees.
\newblock In \emph{International Conference on Artificial Intelligence and
  Statistics (AISTATS)}, 2020.

\bibitem[Stone(1974)]{stone1974earlyCV}
M.~Stone.
\newblock Cross-validatory choice and assessment of statistical predictions.
\newblock \emph{Journal of the American Statistical Association}, 36\penalty0
  (2):\penalty0 111--147, 1974.

\bibitem[Sukkar et~al.(2012)Sukkar, Katz, Zhang, Raunig, and Wyman]{sukkar12}
R.~Sukkar, E.~Katz, Y.~Zhang, D.~Raunig, and B.~T. Wyman.
\newblock Disease progression modeling using hidden {M}arkov models.
\newblock In \emph{2012 Annual International Conference of the IEEE Engineering
  in Medicine and Biology Society}, pages 2845--2848. IEEE, 2012.

\bibitem[Sun et~al.(2019)Sun, Ghosh, Li, Cheng, Mohan, Sampaio, and Hu]{sun19}
Z.~Sun, S.~Ghosh, Y.~Li, Y.~Cheng, A.~Mohan, C.~Sampaio, and J.~Hu.
\newblock A probabilistic disease progression modeling approach and its
  application to integrated {H}untington's disease observational data.
\newblock \emph{JAMIA Open}, 2\penalty0 (1):\penalty0 123--130, 2019.

\bibitem[Tsuboi et~al.(2008)Tsuboi, Kashima, Mori, Oda, and
  Matsumoto]{tsuboi2008training}
Y.~Tsuboi, H.~Kashima, S.~Mori, H.~Oda, and Y.~Matsumoto.
\newblock Training conditional random fields using incomplete annotations.
\newblock In \emph{International Conference on Computational Linguistics
  (Coling 2008)}, pages 897--904, 2008.

\bibitem[Vehtari et~al.(2017)Vehtari, Gelman, and Gabry]{vehtari2017practical}
A.~Vehtari, A.~Gelman, and J.~Gabry.
\newblock Practical {Bayesian} model evaluation using leave-one-out
  cross-validation and {WAIC}.
\newblock \emph{Statistics and Computing}, 27\penalty0 (5):\penalty0
  1413--1432, 2017.

\bibitem[Vershynin(2018)]{vershynin2017hdpBook}
R.~Vershynin.
\newblock \emph{High-dimensional probability: {A}n introduction with
  applications in data science}.
\newblock Cambridge University Press, August 2018.

\bibitem[Wang et~al.(2018)Wang, Zhou, Lu, Maleki, and
  Mirrokni]{wang2018primalDualALOO}
S.~Wang, W.~Zhou, H.~Lu, A.~Maleki, and V.~Mirrokni.
\newblock Approximate leave-one-out for fast parameter tuning in high
  dimensions.
\newblock In \emph{International Conference in Machine Learning (ICML)}, 2018.

\bibitem[Wang et~al.(2014)Wang, Sontag, and Wang]{wang14}
X.~Wang, D.~Sontag, and F.~Wang.
\newblock Unsupervised learning of disease progression models.
\newblock In \emph{Proceedings of the 20th ACM SIGKDD International Conference
  on Knowledge Discovery and Data Mining}, pages 85--94, 2014.

\bibitem[Wilson et~al.(2020)Wilson, Kasy, and
  Mackey]{wilson2020modelSelectionALOO}
A.~Wilson, M.~Kasy, and L.~Mackey.
\newblock Approximate cross-validation: {G}uarantees for model assessment and
  selection.
\newblock In \emph{International Conference on Artificial Intelligence and
  Statistics (AISTATS)}, 2020.

\bibitem[Zheng and Liu(2017)]{zheng2017estimating}
J.~Zheng and H.~X. Liu.
\newblock Estimating traffic volumes for signalized intersections using
  connected vehicle data.
\newblock \emph{Transportation Research Part C: Emerging Technologies},
  79:\penalty0 347--362, 2017.

\end{thebibliography}
\newpage
\appendix

\section{Related work: Approximate CV methods}
\label{app:related_work}

A growing body of recent work has focused on various methods for approximate CV (ACV). As outlined in the introduction, these methods generally take one of two forms. The first is based on taking a single Newton step on the leave-out objective starting from the full data fit, $\hat\param$. This approximation was first proposed by \citet{obuchi2016linearALOO, obuchi2018logisticALOO} for the special cases of linear and logistic regression and first applied to more general models by \citet{beirami2017firstALOO}. While this approximation is generally applicable to any CV scheme (e.g.\ beyond LOOCV) and any model type (e.g.\ to structured models), it is only efficiently applicable to LOOCV for GLMs. In particular, approximating each $\paramlo$ requires the computation and inversion of the leave-out objective's $D \times D$ Hessian matrix. In the case of LOOCV GLMs, this computation can be performed quickly using standard rank-one matrix updates; however, in more general settings, no such convenience applies. 
\\\\
Various works detail the theoretical properties of the NS approximation.
\citet{beirami2017firstALOO,rad2018detailedALOO} provide some of the first bounds on the quality of the NS approximation, but under fairly strict assumptions.
\citet{beirami2017firstALOO} assume boundedness of both the parameter and data spaces, while \citet{rad2018detailedALOO} require somewhat hard-to-check assumptions about the regularity of each leave-out objective (although they successfully verify their assumptions on a handful of problems). 
\citet{koh2019accuracy} prove bounds on the accuracy of the NS approximation with fairly standard assumptions (e.g.\ Lipschitz continuity of higher-order derivatives of the objective function), but restricted to models using $\ell_2$ regularization. \citet{wilson2020modelSelectionALOO} also prove bounds on the accuracy of NS using slightly more complex assumptions but avoiding the assumption of $\ell_2$ regularization. 
More importantly, \citet{wilson2020modelSelectionALOO} also address the issue of model selection, whereas all previous works had focused on the accuracy of NS for model assessment (i.e.\ assessing the error of a single, fixed model). In particular, \citet{wilson2020modelSelectionALOO} give assumptions under which the NS approximation is accurate when used for hyperparameter tuning.
\\\\
Finally, we note that in its simplest form, the NS approximation requires second differentiability of the model objective. \citet{obuchi2016linearALOO,obuchi2018logisticALOO,rad2018detailedALOO,beirami2017firstALOO,stephenson2019sparse}
propose workarounds specific to models using $\ell_1$-regularization. 
More generally, \citet{wang2018primalDualALOO} provide a natural extension of the NS approximation to models with either non-differentiable model losses or non-differentiable regularizers.
\\\\
Again, while these NS methods can be applied to the structured models of interest here, the repeated computation and inversion of Hessian matrices brings their speed into question. To avoid this issue, we instead focus on approximations based on the infinitestimal jackknife (IJ) from the statistics literature \citep{jaeckel1972infinitesimal,efron1981nonparametric}. 
The IJ was recently conjectured as a potential approximation to CV by \citet{koh2017understanding} and then briefly compared against the NS approximation for this purpose by \citet{beirami2017firstALOO}. The IJ was first studied in depth for approximating CV in an empirical and theoretical study by \citet{giordano2018swiss}. 
The benefit of the IJ in our application is that for any CV scheme\footnote{The methods of \citet{giordano2018swiss} apply beyond CV to other ``reweight and retrain'' schemes such as the bootstrap. The methods presented in our paper apply more generally as well, although we do not explore this extension.} and any (differentiable and i.i.d.) model, the IJ requires only a single matrix inverse to approximate all CV folds. 
\citet{koh2019accuracy} give further bounds on the accuracy of the IJ approximation for models using $\ell_2$ regularization. 
As in the case for NS, \citet{wilson2020modelSelectionALOO} give bounds on the accuracy of IJ beyond $\ell_2$ regularized models but with slightly more involved assumptions; \citet{wilson2020modelSelectionALOO} also give bounds on the accuracy of IJ for model selection.
\\\\
Just as for the NS approximation, the IJ also requires second differentiability of the model objective. \citet{stephenson2019sparse} deal with this issue by noting that the methods of \citet{wang2018primalDualALOO} for applying the NS to non-differentiable objectives can be extended to cover the IJ as well. We note that the use of the IJ for model selection for non-differentiable objectives seems to be more complex than for the NS approximation. In particular, \citet[Appendix G]{stephenson2019sparse} show that the IJ approximation can have unexpected and undesirable behavior when used for tuning the regularization parameter for $\ell_1$ regularized models. \citet{wilson2020modelSelectionALOO} resolve this issue by proposing a further modification to the IJ approximation based on proximal operators.

\section{Hidden Markov random fields}
\label{sec:supp_examples}
Here we show that HMMs are instances of (hidden) MRFs. Recall that a MRF models the joint distribution,
\begin{equation}
	\label{eq:supp_mrf}
	-\log p(\xseqone, \zseqone; \param )
		= - \sum_{n \in [N]} \log p(\xseq, \zseq; \Theta) = 
		Z(\param) +
		\sum_{n=1}^{N}
		\left\{ \left[
    			\sum_{t \in [T]} \psi_{t}(\xseqt{t}, \zseqt{t}; \param)
    			\right]
    		+ \left[ 
    			\sum_{\textbf{c} \in \factors} \phi_{\textbf{c}}( \zseqt{\textbf{c}}; \param)
    			\right]
		\right\}.
\end{equation}
\paragraph{Hidden Markov models}
We recover hidden Markov models as described in Section~\ref{sec:background} by setting $\psi_t(\xseqt{t}, \zseqt{t}; \Theta) = \log F(\xseqt{t} \mid \theta_{\zseqt{t}})$, setting $\mathcal{F}$ to the set of all unary and pairwise indices, and defining $\phi_{t, t-1}(\zseqt{t}, \zseqt{t-1}; \Theta) = \log \text{Cat}(\zseqt{t} \mid A_{\zseqt{t-1}})$, and $\phi_{1}(\zseqt{1}) = \log \text{Cat}(\zseqt{1} \mid \pi)$, and $\phi_t(\zseqt{t}) = 0$, for $t \in [T] - 1$. The log normalization constant is $Z(\Theta) = 0$.

\section{Leave structure out cross-validation (LSCV)}
\label{app:LSCV}
\begin{algorithm}
\caption{Structured approximate \lscv}
\label{alg:supp_algorithm}
\begin{algorithmic}[1] 
\REQUIRE $\param_1, \xseqone, \mathcal{O} $\\
\STATE Marginalize over $\zseq$: $\log p(\xseq ; \param) =  \textsc{Marg}(\xseq ;\param)$, $\forall n \in [N]$
\STATE Compute $\log p(\xseqone ; \Theta, \wseqone) = \sum_n w_n \log p(\xseq ; \param)$
\STATE Compute $H = \frac{\partial^2 \log p(\xseqone; \param, \wseqone) + \log p(\param)}{\partial\param \partial \param^\top} \bigg|_{\param = \param_1, \wseqone = \mathbf{1}_\numseq}$ 
\STATE Compute $J = (J_{dn}) := \bigg(\frac{\partial^2 \log p(\xseqone; \param, \wseqone) + \log p(\param)}{\partial\param \partial w_n} \bigg|_{\param = \param_1, \wseqone = \mathbf{1}_\numseq} \bigg)$
\STATE \textbf{for} {$\leftoutinds \in \mathcal{O}$},  $\paramalg(\param_1, \xseqone, \leftoutinds) := \displaystyle \param_1 + \sum_{n \in \mathbf{o}}H^{-1}J_n$  \hfill \# $J_n$ is $n$th column of $J$
\RETURN $\{ \paramalg(\param_1, \xseqone, \leftoutinds) \}_{\leftoutinds \in \mathcal{O}}$
\end{algorithmic}
\end{algorithm}
%
%
\section{Efficient weighted marginalization (\textsc{WeightedMarg}) for chain-structured MRFs}
\label{app:weighted_marg}
For chain-structured pairwise MRFs with discrete structure, we can use a dynamic program to efficiently marginalize out the structure. Assume that $z_t \; \forall t \in [T]$ can take one of $K$ values. Define $\alpha_{1k} = \exp[ w_1 \psi_1(x_1, z_1 = k)]$, and then compute $\alpha_{tk}$ recursively:
\begin{equation}
	\alpha_{t,k} = \sum_{\ell = 1}^K \alpha_{t-1,\ell} \exp\bigg[ w_t \psi_t(x_t, z_{t} = k; \param) + \phi_{t,t-1}(z_t = k, z_{t-1} = \ell) \bigg],
	\label{eq:supp_alpha_A}
\end{equation}
if using weighting scheme (A) from Equation~\cref{eq:mrf_cv_data_A} or,
\begin{equation}
	\alpha_{t,k} = \sum_{\ell = 1}^K \alpha_{t-1,\ell} \exp\bigg[ w_t \psi_t(x_t, z_{t} = k; \param) + w_t w_{t-1} \phi_{t,t-1}(z_t = k, z_{t-1} = \ell) \bigg],
	\label{eq:supp_alpha_B}
\end{equation}
if using weighting scheme (B) from Equation~\cref{eq:mrf_cv_data_B}.
Then, for either (A) or (B), we have $p(\xseqone ; \param, \wseqone) = \sum_{k=1}^K \alpha_{Tk}$. When $\wseqone = \mathbf{1}_T$ we recover the empirical risk minimization solution. 
As is the case for non-weighted models, this recursion implies that $p(\xseqone ; \Theta, \wseqone)$ is computable in $O(TK^2Q)$ time instead of the usual $O(T^KQ)$ time required by brute-force summation (recall $Q$ is the time required to evaluate one local potential). 
Likewise, we can also compute the derivatives needed by \cref{alg:main_algorithm} in $O(TK^2Q)$ time either by manual implementation or automatic differentiation tools~\citep{biggs2000AD}.
\section{Equivalence of weighting (A) and (B) for leave-future-out for chain-structured graphs}
\label{app:chain_equivalence}
As noted in the main text, weighting schemes (A) and (B) are equivalent when the graph is chain structured. Formally,
\begin{prop}
\label{prop:chain_equivalence}
	Consider a chain-structured pairwise MRF with ordered indices $t$ on the chain (such as an HMM). 
	Weighting styles (A) and (B) above are equivalent for \emph{leave-future-out} CV. That is, choose $\leftoutinds = \{T', \ldots, T\}$ for some $T' \in [T-1]$ (i.e., indices that are in the ``future'' when interpreted as time). Then set $\forall t \in \leftoutinds, \wt = 0$ and $\forall t \in [\numsteps] - \leftoutinds, \wt = 1$. 
\end{prop} 
This result does not hold generally beyond chain-structured graphical models -- consider a four-node ``ring'' graph in which node $t$ is connected to nodes $t-1$ and $t+1$ (mod 4) for $t = 0, \dots, 3$. Weighting scheme (B) produces a distribution that is chain-structured over three nodes, whereas (A) produces a distribution without such conditional independence properties. We now prove \cref{prop:chain_equivalence}.
\begin{proof}

Recall that for a chain structured graph, we can write:
$$
	p(x,z) = p(x \mid z) p(z_1) \prod_{t=2}^T p(z_t \mid z_{t-1}).
$$
Let $\mathbf{o} = \{ T', T'+1, \dots, T\}$ for some $T' < T$; that is, we are interested in dropping out time steps $T', \dots, T$. For weighting scheme (A) (\cref{eq:mrf_cv_data_A}), we drop out only the observations, obtaining:
$$ 
p_A(x,z ; w_{\mathbf{o}}) = \left( \prod_{t=1}^{T'-1} p(x_t \mid z_t) \right) p(z_1) \prod_{t=2}^T p(z_t \mid z_{t-1}).
$$
When we sum out all $z$ to compute the marginal $p_A (x ; w_{\mathbf{o}})$, we can first sum over $z_T, \dots, z_{T'}$. As $\sum_{z_t} p(z_t \mid z_{t-1}) = 1$ for any value of $z_{T-1}$, we obtain:
$$
	p_A(x; w_\mathbf{o}) = \sum_{z_1, \dots, z_{T'-1}} \left( \prod_{t=1}^{T'-1} p(x_t \mid z_t) \right) p(z_1) \prod_{t=2}^{T'-1} p(z_t \mid z_{t-1}),
$$
which is exactly the formula for $p_B(x ; w_\mathbf{o})$, the marginal likelihood from following weighting scheme (B) (\cref{eq:mrf_cv_data_B}), in which we drop out both the $x_t$ and $z_t$ for $t \not\in \mathbf{o}$.
\end{proof}
\section{Conditional random fields}
\label{sec:supp_crf}
Conditional random fields assume that the \emph{labels} $\zseqone$ are observed and model the conditional distribution $p(\zseqone \mid \xseqone; \param)$. While more general dependencies between $\xseqone$ and $\zseqone$ are possible a commonly used variant~\citep{ma2016end, lample2016neural} captures the conditional distribution of the joint defined in Equation~\cref{eq:supp_mrf}. Note,
\begin{equation}
\begin{split}
\log p(\zseq \mid \xseq; \Theta) &= \log p(\xseq, \zseq; \Theta) - \log p(\xseq; \Theta) \\
 &= - Z(\Theta) + 
    			\sum_{t \in [T]} \psi_{t}(\xseqt{t}, \zseqt{t}; \param)
    		+   \sum_{\textbf{c} \in \factors} \phi_{\textbf{c}}( \zseqt{\textbf{c}}; \param) \\
 &  + Z(\Theta) - \int_{\zseq} 
    			\sum_{t \in [T]} \psi_{t}(\xseqt{t}, \zseqt{t}; \param)
    		+   \sum_{\textbf{c} \in \factors} \phi_{\textbf{c}}( \zseqt{\textbf{c}}; \param) 	d\zseq	
\end{split}
\end{equation}
Defining, $Z(\xseq; \Theta) = -\int_{\zseq} 
    			\sum_{t \in [T]} \psi_{t}(\xseqt{t}, \zseqt{t}; \param)
    		+   \sum_{\textbf{c} \in \factors} \phi_{\textbf{c}}( \zseqt{\textbf{c}}; \param) 	d\zseq$, then gives us the following conditional distribution,
 \begin{equation}
 -\log p(\zseqone \mid \xseqone; \param) = \sum_{n=1}^N \left\{ Z(\xseq; \param) + \sum_{t \in [T]} \psi_{t}(\xseqt{t}, \zseqt{t}; \param) +  \sum_{\textbf{c} \in \factors} \phi_{\textbf{c}}(\zseqt{\textbf{c}}; \param)\right\}.
\end{equation}
Note that $Z(\xseq; \param)$ is an observation specific negative normalization constant. 
\section{CV for conditional random fields}
\label{app:crf_cv}
Analogously to the MRF case, we have two variants for CRFs --- \lscv and \lwcv. While \lscv is frequently used in practice, for example, \citep{decaprio2007conrad}, we are unaware of instances of \lwcv in the literature. Thus, while we derive approximations to both CV schemes, our CRF-based experiments in \cref{sec:exp} only use \lscv.

\subsection{\lscv for CRFs}
Leave structure out CV is analogous to the MRF case and is detailed in \cref{alg:supp_algorithm2}, where $\log \tilde{p}(\zseq, \xseq; \Theta) := \sum_{t \in [T]} \psi_{t}(\xseqt{t}, \zseqt{t}; \param)
    		+   \sum_{\textbf{c} \in \factors} \phi_{\textbf{c}}( \zseqt{\textbf{c}}; \param)$. Since all input, label pairs $\{\xseq, \zseq\}$ are independent, $\log p(\zseqone \mid \xseqone ; \Theta, \wseqone)$ is just a weighted sum across $n$ and the losses $ - \log p(\zseq \mid \xseq ;\hat\param(\wseq{\{n\}}) )$ and $-\log p(\zseq \mid \xseq ; \paramij(\wseq{\{n\}}))$ do not depend on $[N] - n$.
\begin{algorithm}[ht]
\caption{Structured approximate cross-validation (\lscv) for CRFs}
\label{alg:supp_algorithm2}
\begin{algorithmic}[1] 
\REQUIRE $\param_1, \xseqone, \zseqone, \mathcal{O} $\\
\STATE Compute $Z(\xseq; \Theta) = -\textsc{Marg}(\xseq ;\param)$, $\forall n \in [N]$
\STATE Compute $\log p(\zseqone \mid \xseqone ; \Theta, \wseqone) = \sum_n w_n \big[Z(\xseq; \Theta) + \log \tilde{p}(\zseq, \xseq; \Theta) \big]$
\STATE Compute $H = \frac{\partial^2 \log p(\zseqone \mid \xseqone; \param, \wseqone) + \log p(\param)}{\partial\param \partial \param^\top} \bigg|_{\param = \param_1, \wseqone = \mathbf{1}_\numseq}$ 
\STATE Compute matrix $J := (J_{dn})= \bigg(\frac{\partial^2 \log p(\zseqone \mid \xseqone; \param, \wseqone) + \log p(\param)}{\partial\param \partial w_n} \bigg|_{\param = \param_1, \wseqone = \mathbf{1}_\numseq}\bigg)$
\STATE \textbf{for} {$\leftoutinds \in \mathcal{O}$},  $\paramalg(\param_1, \xseqone, \zseqone, \leftoutinds) := \displaystyle \param_1 + \sum_{n \in \mathbf{o}}H^{-1}J_n$  \hfill \# $J_n$ is $n$th column of $J$
\RETURN $\{ \paramalg(\param_1, \xseqone, \zseqone, \leftoutinds) \}_{\leftoutinds \in \mathcal{O}}$
\end{algorithmic}
\end{algorithm}

\subsection{\lwcv for CRFs}
Leave within structure out for CRFs again comes with a choice of weighting scheme. Given a single input, label pair $\xseqone, \zseqone$, the $z_t$ are the outputs at location $t$, and the $x_t$ are the corresponding inputs. A form of CV arises when we drop the outputs $z_t$, for $t \in \mathbf{o}$. This gives us weighting scheme (C), 
\begin{equation}
	\label{eq:crf_cv_data_C}
	\begin{split}
	\hat\param(\wseqone)
		= &\argmin{\param}
			Z(\param, \wseqone, \xseqone) \\
		&+ \left[
    			\sum_{t \in [T]} \wt \psi_{t}(\xseqonet{t}, \zseqonet{t}; \param) + (1 - \wt) \int_{\zseqonet{t}}\psi_{t}(\xseqonet{t}, \zseqonet{t}; \param)	\; d\zseqonet{t}
    			\right] \\
    		&+ \left[ 
    			\wt\sum_{\textbf{c} \in \factors} \phi_{\textbf{c}}( \zseqonet{\textbf{c}}; \param) + (1-\wt)\int_{\zseqonet{t}}\sum_{\textbf{c} \in \factors}\phi_{\textbf{c}}( \zseqonet{\textbf{c}}; \param)\; d\zseqonet{t}
    			\right]
		- \log p(\param)
		.
	\end{split}
\end{equation}
For linear chain structured CRFs with discrete outputs $\zseqone$ a variant of the forward algorithm can be used to efficiently compute $Z(\param, \wseqone, \xseqone)$ as well as the marginalizations over $\{z_t \mid t \in \mathbf{o}\}$ required by \cref{eq:crf_cv_data_C}. See~\citet{bellare2007learning, tsuboi2008training} for details. \cref{alg:supp_algorithm3} summarizes the steps involved.

\begin{algorithm}
\caption{Approximate leave-within-structure-out cross-validation for CRFs}
\label{alg:supp_algorithm3}
\begin{algorithmic}[1] 
\REQUIRE $\param_1, \xseqone, \zseqone, \mathcal{O} $\\
\STATE Compute \emph{unweighted} marginalization over $\zseqone$, $Z(\xseqone; \Theta) = -\textsc{Marg}(\xseq ;\param)$, $\forall n \in [N]$
\STATE Compute \emph{weighted} marginalization over $\zseqone$: $Z(\xseqone; \param, \wseqone) =  \textsc{WeightedMarg}(\xseqone, \param, \wseqone)$. 
\STATE Compute $\log p(\zseqone \mid \xseqone ; \param) =  Z(\xseqone; \param, \wseqone) + Z(\xseq; \Theta)$
\STATE Compute $H = \frac{\partial^2 \log p(\xseqone; \param, \wseqone) + \log p(\param)}{\partial\param \partial \param^\top} \bigg|_{\param = \param_1, \wseqone = \mathbf{1}_\numsteps}$ 
\STATE Compute matrix $J := (J_{dt}) = \bigg(\frac{\partial^2 \log p(\xseqone; \param, \wseqone) + \log p(\param)}{\partial\param \partial \wt} \bigg|_{\param = \param_1, \wseqone = \mathbf{1}_\numsteps}\bigg)$
\STATE \textbf{for} $\leftoutinds \in \mathcal{O}$, \textbf{do:} \, $\paramalg(\param_1, \xseqone, \leftoutinds) := \displaystyle \param_1 + \sum_{t \in \leftoutinds}H^{-1}J_t$ \hfill \# $J_t$ is $t$th column of $J$ 
\RETURN $\{ \paramalg(\param_{\mathbf{1}}, \xseqone, \leftoutinds) \}_{\leftoutinds \in \mathcal{O}}$
\end{algorithmic}
\end{algorithm}

\section{Computational cost of one Newton-step-based ACV}
\label{app:NS}
Recall that we define $M$ to be the cost of one marginalization over the latent structure $z$ and noted above that the cost of computing the Hessian via automatic differentiation is $O(M)$. For the Newton step (NS) approximation, recall that we need to compute a different Hessian for each fold $\leftoutinds$. While this can be avoided using rank-one update rules in the case of leave-one-out CV for generalized linear models, this is not the case for the CV schemes and models considered here.
Thus, to use the Newton step approximation here, we require $O(M\abs{\mathcal{O}})$ time to compute all needed Hessians. Compared to the $O(M)$ time spent computing Hessians by our algorithms, the Newton step is significantly more expensive. For this reason, we do not consider Newton step based approximations here.
\section{Comparison of approximations afforded by one Newton-step-based and IJ based ACV}
\label{app:IJ_v_NS}
We revisit the \lwcv experiments in time varying Poisson processes described in \cref{sec:exp}. We agian focus on the $\numsteps_{sub} = 10{,}000$ subset of observations, plotted in the top panel of \cref{fig:dodgers_withinseq}. In \cref{app:fig_NS_v_IJ} we compare estimates provided by ACV based on one NS to those provided by IJ based ACV. The left plot depicts i.i.d \lwcv and the right depicts contiguous \lwcv when $m=10\%$ of the subset is held out. Similar results hold for $m = 2\%$ and $m=5\%$. For each of $\numfolds = 10$ folds and for each point $\xseqonet{t}$ left out in each fold, we plot a red dot with the NS bases approximate fold loss as its horizontal coordinate and our IJ based approximation as its vertical coordinate. We can see that every point lies close to the dashed black $x=y$ line; that is, the quality of the two approximations largely agree across the thousands of points in each plot.
\begin{figure}[ht]
\centering
\includegraphics[width=1\textwidth]{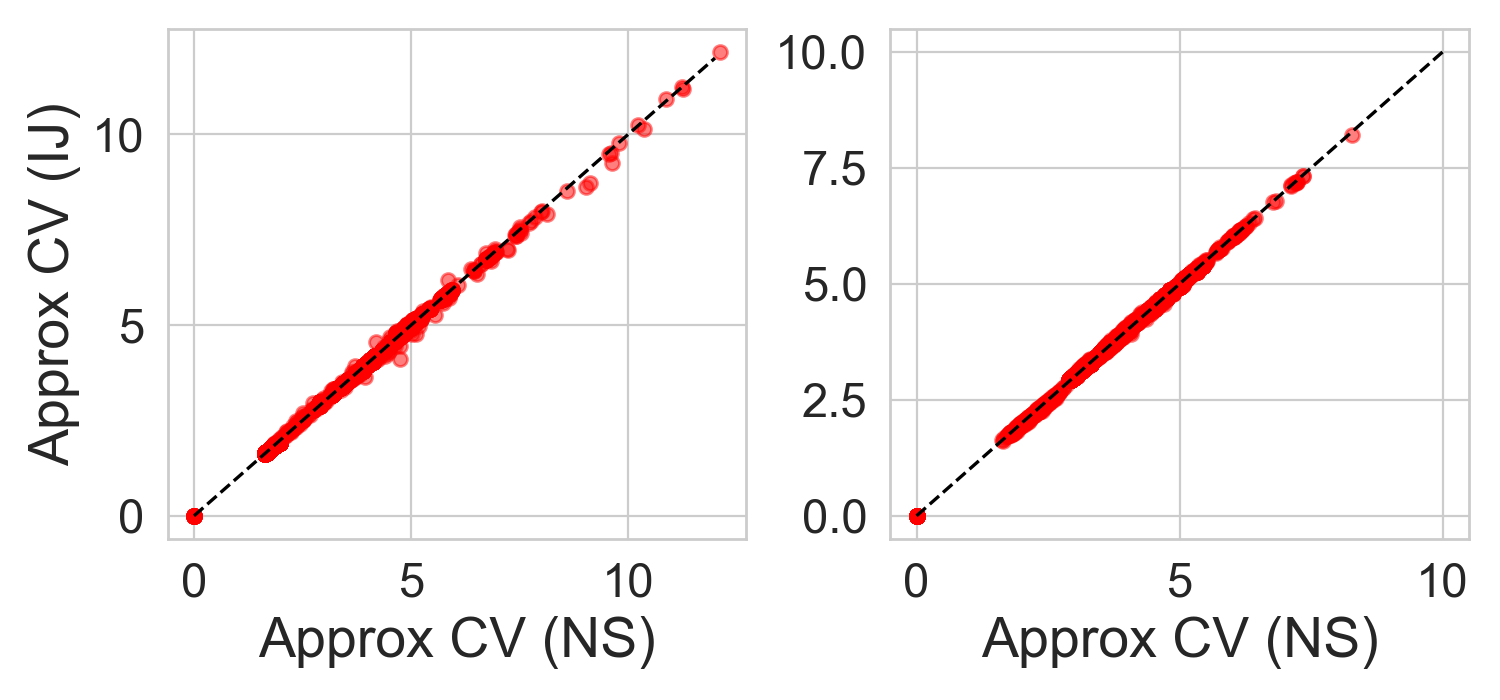}
\caption{Comparison of NS and IJ approximate \lwcv for time-varying Poisson processes.   Scatter plots comparing NS based ACV loss (horizontal axis) at each point in each fold (red dots) to IJ based ACV loss (vertical axis). Black dashed line shows perfect agreement. Left plot containts i.i.d. \lwcv results and the right plot contains contiguous \lwcv results.}
\label{app:fig_NS_v_IJ}
\end{figure}

\section{Derivation of IJ approximations}
\label{app:IJ_derivations}

In all cases considered here (i.e., the ``exchangeable'' leave-one-out CV considered by previous work or the more structured variants for chain-structured or general graph structured models) can be derived similarly. In particular, once we have derived the relevant weighted optimization problem for each case, the derivation of the IJ approximation is the same. Let the relevant weighted optimization problem be defined for $w \in \R^T$:
$$ \paramperturb{w} := \argmin{\param \in \R^D} F(\param, w),$$
where $F$ is some objective function with $F(\cdot, \mathbf{1}_T)$ corresponding to the ``full-data'' fit (i.e., without leaving out any data). We now follow the derivation of the IJ in \citet{giordano2018swiss}. The condition that $\paramone$ is an exact optimum is:
$$ 
\frac{\partial F}{\partial \param}\at{\paramone, \mathbf{1}_T} = 0.
$$
If we take a derivative with respect to $w_t$:
$$ 
\frac{\partial^2 F}{\partial \param \partial \param^T}\at{\paramone, \mathbf{1}_T} \frac{d \param}{d w_t}\at{\paramone, \mathbf{1}_T} + \frac{\partial^2 F}{\partial \param \partial w_t}\at{\paramone, \mathbf{1}_T} \frac{d w_t}{d w_t}\at{\paramone, \mathbf{1}_T} = 0.
$$
Noting that $d w_t / d w_t = 1$ and solving for $d \param / dw_t$:
\begin{equation}
	\frac{d \param}{d w_t}\at{\paramone, \mathbf{1}_N} = 
	   - \left( \frac{\partial^2 F}{\partial \param \partial \param^T}\at{\paramone, \mathbf{1}_T} \right)^{-1} \frac{\partial^2 F}{\partial \param \partial w_t}
\end{equation}
Thus we can form a first order Taylor series of $\paramperturb{w}$ in $w$ around $w = \mathbf{1}_N$ to approximate:
$$
 \paramperturbij{w} \approx \paramone - \sum_{t=1}^T \left( \frac{\partial^2 F}{\partial \param \partial \param^T}\at{\paramone, \mathbf{1}_T} \right)^{-1} \frac{\partial^2 F}{\partial \param \partial w_t} (1 - w_t).
$$
Specializing this last equation to the various $F$ and weight vectors $w$ of interest derives each of our ACV algorithms.

\section{Inexact optimization}
\label{app:inexact_optimization}

We prove here a slightly more general version of \cref{prop:inexact_optimization} that covers both \lwcv and \lscv, as well as arbitrary loss functions $\ell$. To encompass both in the same framework, let $\wseqone_n \in \R^T$ be weight vectors for each structured object $n = 1, \dots, N$. Our weighted objective will be:
$$
	\hat\param(\wseqone) = \argmin{\param \in \R^D} \sum_{n=1}^N \log p(\data_n ; \param, \wseqone_n) + p(\param),
$$
where $\data = \{\data_1, \dots, \data_N\}$ denotes the collection of all observed structures; i.e., each $\data_n$ may be a sequence of observations $x_n$ for a HMM or observed outputs and inputs $x_n, z_n$ for a CRF. 
Let $\hat\param(\mathbf{1}_{NT})$ be the solution to this problem with $w_{nt} = 1$ for all $n$ and $t$. 
We assume that we are interested in estimating the exact out-of-sample loss for some generic loss $\ell$ by using exact CV, 
$
	\losscv := (1/\abs{\mathcal{O}}) \sum_{\leftoutinds} \ell(\data_{\leftoutinds}, \data_{-\leftoutinds}, \hat\param(\wseqo))
$; 
e.g., we may have $\ell(\data_\leftoutinds, \data_{-\leftoutinds}, \hat\param(\wseqo)) = - \log p(x_\leftoutinds \mid x_{[\numsteps] -\leftoutinds} ; \hat\param(\wseqo))$ in the case of a HMM with $N = 1$.
Notice here that $\leftoutinds \subset [N] \times [\numsteps]$ indexes arbitrarily across structures. We can now state a modified version of \cref{assum:inexact_optimization}.
\begin{assumption}
\label{assum:app-inexact_optimization}
  Let $B \subset \R^D$ be a ball centered on $\paramonent$ and containing $\paramstep{S}$. Then the objective $\sum_n \log p(x_n ; \param, \mathbf{1}_\numsteps) + p(\param)$ is strongly convex with parameter $\lambda_{
  \mathrm{min}}$ on $B$. Additionally, on $B$, the derivatives $g_{nt}(\param) := \partial^2 \log p(x_n ; \param, \wseqone_n) / \partial \param \partial w_{nt}$ are Lipschitz continuous with constant $L_g$ for all $n,t$ and the inverse Hessian of the objective is Lipschitz with parameter $L_{Hinv}$. Finally, on $B$, $\ell(\data_\leftoutinds, \data_{-\leftoutinds}, \param)$ is a Lipschitz function of $\param$ with parameter $L_\ell$ for all $\leftoutinds$.
\end{assumption}

We now prove our more general version \cref{prop:inexact_optimization}.
\begin{prop}
  Take \cref{assum:app-inexact_optimization}. Then the approximation error of $\lossij(\paramstep{S})$ is bounded by:
  \begin{equation} 
  		| \lossij(\paramstep{S}) - \losscv)| \leq C \epsparam + \epsij, 
  	\end{equation}
  where $C$ is given by
  $$ \left( L_\ell + \frac{L_\ell L_g}{\lambda_{\mathrm{min}}} + \frac{L_\ell L_{Hinv}}{\abs{\mathcal{O}}} \sum_\mathbf{o} \n{\sum_{t \in \mathbf{o}} g_{nt} (\paramonent)}_2 \right).$$
\end{prop}
\begin{proof}
By the triangle inequality:
\begin{align*} 
  	\abs{\lossij(\paramstep{S}) - \losscv} & \leq \\
    & \abs{\lossij(\paramstep{S}) - \lossij(\paramonent)} \\
    & + \abs{\lossij(\paramonent) - \losscv}.
  \end{align*}
 The second term is just the constant $\epsij$. Now we just need to bound the first term using our Lipschitz assumptions. We have, by the triangle inequality
  \begin{align*}
  	  & \abs{\lossij(\paramonent) - \lossij(\paramstep{S})} \\
  	  & \; \leq \frac{1}{\abs{\mathcal{O}}} \sum_{\mathbf{o}} \left\lvert \ell\left(\data_\leftoutinds, \data_{-\leftoutinds}, \paramonent + H^{-1} (\paramonent) \sum_{t \in \mathbf{o}} g_{nt}(\paramonent) \right) \right. \\
  	  & \left. \quad\quad\quad\quad - \ell \left(\data_\leftoutinds, \data_{-\leftoutinds}, \paramstep{S} + H^{-1} (\paramstep{S}) \sum_{t \in \mathbf{o}} g_{nt}(\paramstep{S}) \right) \right\rvert .
  \end{align*}
  Continuing to apply the triangle inequality and our Lipschitz assumptions:
\begin{align*}
	&\leq \frac{L_\ell}{\abs{\mathcal{O}}} \sum_{\mathbf{o}} \left( \n{ \paramonent - \paramstep{S}}_2 + \n{ H^{-1}(\paramonent) \sum_{t \in \mathbf{o}} g_{nt}(\paramonent) - H^{-1}(\paramstep{S}) \sum_{t \in \mathbf{o}} g_{nt} (\paramstep{S})}_2 \right) \\
	&\leq L_\ell \epsparam + \frac{L_\ell}{\abs{\mathcal{O}}} \sum_\mathbf{o} \n{H^{-1}(\paramstep{S}) \sum_{t \in \mathbf{o}} \left( g_{nt}(\paramonent) - g_{nt} (\paramstep{S}) \right)}_2  \\
	& \quad\quad\quad\quad + \frac{L_\ell}{\abs{\mathcal{O}}} \sum_\mathbf{o} \n{ \left( H^{-1}(\paramonent) - H^{-1} (\paramstep{S}) \right) \sum_{t \in \mathbf{o}} g_{nt}(\paramonent)}_2  \\
	&\leq \left( L_\ell + \frac{L_\ell L_g}{\lambda_{\mathrm{min}}} + \frac{L_\ell L_{Hinv}}{\abs{\mathcal{O}}} \sum_\mathbf{o} \n{\sum_{t \in \mathbf{o}} g_{nt} (\paramonent)}_2 \right) \epsparam.
\end{align*} 
  Defining the term in the parenthesis as $C$ finishes the proof.
  
\end{proof}

As noted after the statement of \cref{prop:inexact_optimization} in the main text, $(1/\abs{\mathcal{O}})  \sum_{\leftoutinds \in \mathcal{O}} \n{\sum_{t \in \mathbf{o}} g_{nt} (\paramonent)}_2$ may depend on $T$, $N$ or $\mathcal{O}$, but we expect it to converge to a constant given reasonable distributional assumptions on the data. To build intuition, we consider the case of leave-one-out CV for generalized linear models, where we observe a dataset of size $N > 1$ and have $T = 1$. In particular,  we have $\log(x_n, y_n; \param) = f(x_n^T\param, y_n)$, where $x_n \in \R^D$ are the covariates and $y_n \in \R$ are the responses. In this case, $g_{nt} = \dnone x_n$, where $\dnone = d f(z) / dz\at{z = x_n^T \paramone}$. 
Then, given reasonable distributional assumptions on the covariates and some sort of control over the derivatives $\dnone$, we might suspect that $(1/N) \sum_n \abs{\dnone} \n{x_n}_2$ will converge to a constant.  As an example, we consider logistic regression with sub-Gaussian data, for which we can actually prove high-probability bounds on this sum.
\begin{defn} \label{defineSubGaussian} [e.g., \citet{vershynin2017hdpBook}]
For $c_x > 0$, a random variable $V$ is $c_x$-\emph{sub-Gaussian} if
$$
  	\mathbb{E}\left[ \exp\left(V^2 / c_x^2 \right) \right] \leq 2.
$$
\end{defn}
\begin{prop}
 For logistic regression, assume that the components of the covariates $x_{nd}$ are i.i.d.\ from a zero-mean $c_x$-sub-Gaussian distribution for $d = 1, \dots, D$. Then we have that, for any $t \geq 0$:
 \begin{equation}
 	\mathrm{Pr}\left[ \abs{ \frac{1}{N} \sum_{n=1}^N  \n{\nabla f(\paramone, x_n)}_2 - \sqrt{D}} \geq t \right] \leq \exp\left[ -C \frac{N t^2}{c_x^2} \right],
 \end{equation}
 where $C > 0$ is some global constant, independent of $N, D,$ and $c_x$.
\end{prop}
\begin{proof}
	First, we can use the fact that $\n{\nabla f(\paramone, x_n)}_2 \leq \n{x_n}_2$, as for logistic regression, $\abs{\dnone} \leq 1$. Next, we can use the fact that $\n{x_n}_2 - \sqrt{D}$ is a zero-mean sub-Gaussian random variable by Theorem 3.1.1 of \citet{vershynin2017hdpBook}. We can then apply Hoeffding's inequality \citep[Theorem 2.6.3]{vershynin2017hdpBook} to complete the proof.
\end{proof}


\section{Experimental details}
\label{sec:supp_details}
We provide further experimental details in this section.

\subsection{Time varying Poisson processes}
\label{subsec:dodger}

We briefly summarize the time-varying Poisson process model from \cite{ihler2006adaptive} here.
Our data is a time series of loop sensor data collected every five minutes over a span of 25 weeks from a section of a freeway near a baseball stadium in Los Angeles. In all, there are 50,400 measurements of the number of cars on that span of the freeway.
\citeauthor{ihler2006adaptive} analyze the resulting time series of counts $\xseqone$ to detect the presence or absence of an event at the stadium.
Following their model, we use a background Poisson process with a time varying rate parameter $\lambda_t$ to model \emph{non-event} counts, $x_{b_t} \sim \text{Poisson}(\lambda_t)$. To model the daily variation apparent in the data, we define $\lambda_t \triangleq \lambda_o\delta_{d_t}$, where $d_t$ takes one of seven values, each corresponding to one day of the week and $[\delta_1 / 7, \ldots, \delta_7 / 7] \sim \text{Dir}(1, \ldots, 1)$. We use binary latent variables $z_t$ indicate the presence or absence of an event and assume a first order Markovian dependence, $z_t \mid z_{t-1} \sim A_{z_t-1}$. Next,  $z_t=0$ indicates a non-event at time step $t$ and the observed counts are generated as $x_t = x_{b_t}$. An event at time step $t$ corresponds to $z_t=1$ and $x_t =  x_{b_t} + x_{e_t}$, and  
$
x_{e_t} \sim \text{NegBinomial}(x_{e_t} \mid a, b/(1 + b)),
$ where $x_{e_t}$ are unobserved excess counts resulting from the event. We place Gamma priors on $\lambda_0, a, b$ and Beta priors on $A_{00}$ and $A_{11}$, and learn the MAP estimates of the parameters $\Theta = \{\lambda_0, \delta_1, \ldots, \delta_7, a, b, A\}$ while marginalizing $x_{e_t}$ and $z_1, \ldots, z_T$.  We refer the interested reader to \cite{ihler2006adaptive} for further details about the model and data.

\paragraph{Contiguous \lwcv.} In contiguous \lwcv we leave out contiguous blocks from a time series. To drop $m\%$ of the data, we sample an index $t$ uniformly at random from $[\lfloor mT/100 \rfloor + 1, \ldots, T]$ and set $\leftoutinds = \{t - \lfloor mT/100 \rfloor, \ldots t\}$. 

\paragraph{Numerical values from \cref{fig:dodgers_withinseq}}
In \cref{tab:dodgers} we present an evaluation of the LWCV approximation quality for time-varying Poisson processes. The results presented are a numerical summary of the results visually illustrated in \cref{fig:dodgers_withinseq}. 
\begin{table}[ht]
\center
	\begin{tabular}
	{|l||c|c|c|}
	\hline
	& 2 \% & 5\% & 10 \% \\
	\hline
	i.i.d & $0.005 \pm 0.009$ & $0.006 \pm 0.01$ & $0.006 \pm 0.005$ \\
	contiguous & $0.003 \pm 0.003$ & $0.007 \pm 0.02$ & $0.007 \pm 0.006$ \\
	\hline
	\end{tabular}
	\vskip 1mm
	\caption{Evaluation of approximate LWCV for time-varying Poisson processes. Mean ACV relative error, $|acv - cv| / cv$ and two standard deviations, over ten folds with $T=10000$. The numbers summarize the scatter plots in the lower left six panels of \cref{fig:dodgers_withinseq}. The column headers indicate the percentage of data in the held out fold. }
\label{tab:dodgers}
\end{table}
\cref{tab:dodgers_timing} presents the wall clock time numbers plotted in the lower right panels of \cref{fig:dodgers_withinseq}.
\begin{table}[ht]
\center
	\begin{tabular}
	{|c||c|c|c|c|c|c|}
	\cline{1-7}
	\multicolumn{4}{|c}{i.i.d} &
	\multicolumn{3}{|c|}{contiguous}\\
	\cline{1-7}
	\hline
	T & ACV & ACV (NS) & Exact CV & ACV & ACV (NS) & Exact CV \\
	\hline
5000 & 1.1 mins & 10.5 hours & 61.1 hours & 1.3 mins & 10.5 hours & 61.3 hours\\
10000 & 2.2 mins & 19.9 hours & 185.8 hours & 2.4 mins & 19.9 hours & 182.4 hours
\\
50000 & 11.0 mins & 98.6 hours & 682.2 hours & 10.6 mins & 99.1 hours & 683.9 hours
\\
\hline
	\end{tabular}
	\vskip 1mm
	\caption{Wall clock time from the two lower right panels in \cref{fig:dodgers_withinseq} at $T=50000$ and with $m \% = 10 \%$ of the data in the held out fold.}
\label{tab:dodgers_timing}
\end{table}

\subsection{Neural CRF}
\label{subsec:neural_crf}
We employed a bi-directional LSTM model with a CRF output layer. We used a concatenation of a 300 dimensional Glove word embeddings~\citep{pennington2014glove} and a character CNN~\citep{ma2016end} based character representation. We employed variational dropout with a dropout rate of $0.25$. The architecture is detailed below.
\begin{verbatim}
LSTMCRFVD(
  (dropout): Dropout(p=0.25, inplace=False)
  (char_feats_layer): CharCNN(
    (char_embedding): CharEmbedding(
      (embedding): Embedding(96, 50, padding_idx=0)
      (embedding_dropout): Dropout(p=0.25, inplace=False)
    )
    (cnn): Conv1d(50, 30, kernel_size=(3,), stride=(1,), padding=(2,))
  )
  (word_embedding): Embedding(2196016, 300)
  (rnn): StackedBidirectionalLstm(
    (forward_layer_0): AugmentedLstm(
      (input_linearity): Linear(in_features=330, out_features=200, bias=False)
      (state_linearity): Linear(in_features=50, out_features=200, bias=True)
    )
    (backward_layer_0): AugmentedLstm(
      (input_linearity): Linear(in_features=330, out_features=200, bias=False)
      (state_linearity): Linear(in_features=50, out_features=200, bias=True)
    )
    (forward_layer_1): AugmentedLstm(
      (input_linearity): Linear(in_features=100, out_features=200, bias=False)
      (state_linearity): Linear(in_features=50, out_features=200, bias=True)
    )
    (backward_layer_1): AugmentedLstm(
      (input_linearity): Linear(in_features=100, out_features=200, bias=False)
      (state_linearity): Linear(in_features=50, out_features=200, bias=True)
    )
    (layer_dropout): InputVariationalDropout(p=0.25, inplace=False)
  )
  (rnn_to_crf): Linear(in_features=100, out_features=9, bias=True)
  (crf): ConditionalRandomField()
)
\end{verbatim}

\paragraph{Training}  We used Adam for optimization. Following the recommendation of ~\cite{reimers2017optimal} we used mini-batches of size $31$~\cite{reimers2017optimal}. We employed early stopping by monitoring the loss on the validation set. Freezing all but the CRF layers we further fine-tuned only the CRF layer for an additional 60 epochs. In~\cref{supp:crf_v_time} we plot the mean absolute approximation error in the held out probability under exact CV and our approximation across all 500 folds as a function of (wall clock) time taken by the optimization procedure. 
\begin{figure}
\centering
	\includegraphics[width=0.48\textwidth]{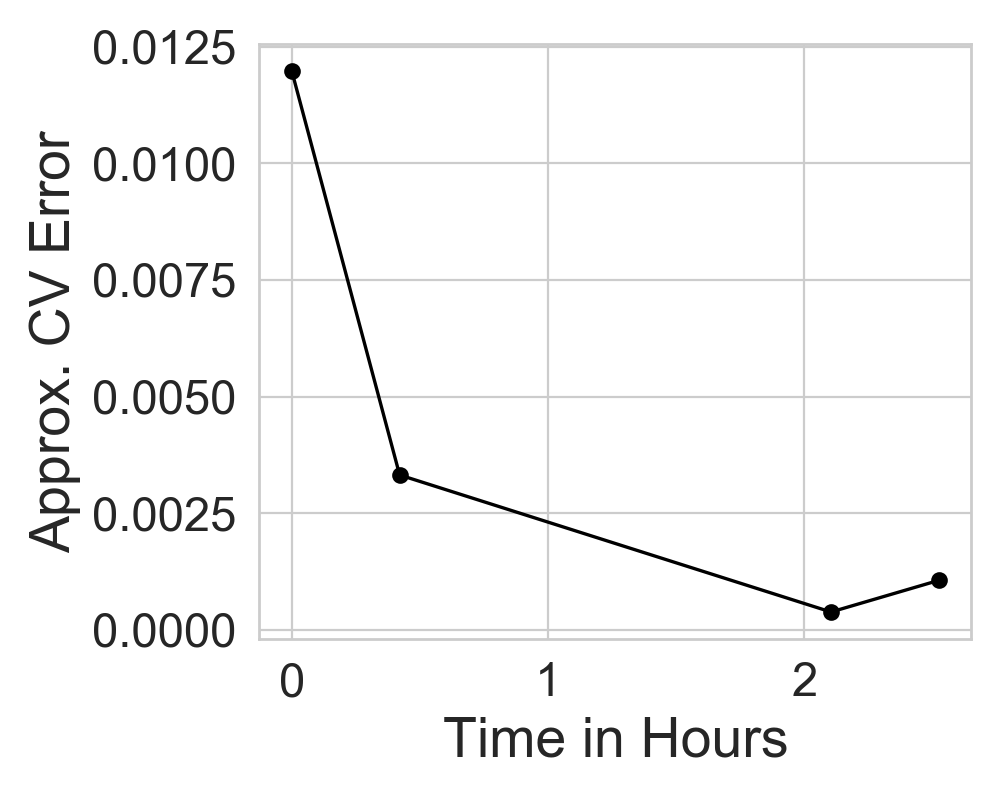}
	\includegraphics[width=0.48\textwidth]{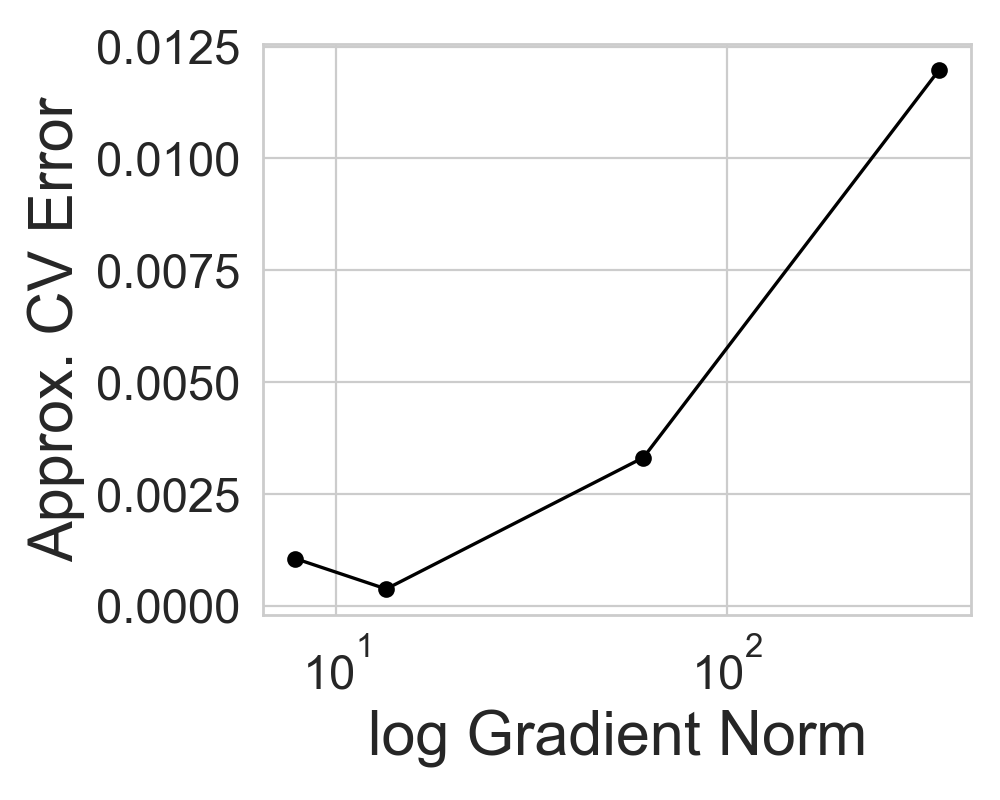}
	\caption{\emph{(Left panel)} Error in our approximation relative to exact CV averaged across folds, as a function of wall clock time. \emph{(Right panel)} Error in our approximation relative to exact CV and averaged across folds, as a function of log gradient norm in the optimization procedure.}
	\label{supp:crf_v_time}
\end{figure}

\subsection{Philadelphia crime experiment}
\label{subsec:philly}

Our crime data comes from \url{opendataphilly.org}, where the Philadelphia Police Department publicly releases the time, type, and location of every reported time.
For each census tract, we have a latent label $z_{t} \in \{-1,1\},$ and model the number of reported crimes $x_{t}$ with a simple Poisson mixture model: $x_{t} \vert z_{t} \sim \text{Poisson}(\lambda_{z_{t}})$ where $\lambda_{-1}, \lambda_{1} > 0$ are the unknown mean levels of crime in low- and high-crime areas, respectively. 
Since we might expect adjacent census tracts to be in the same latent state, we model the $z_{t}$'s with an MRF so that 
$$
\log{p(\xseqone, \zseqone; \param)} = \sum_{t}{[-\lambda_{z_{t}} + x_{t}\log{\lambda_{z_{t}} - \log(x_t!) ]}} + \beta \sum_{t}{\sum_{t' \in \Gamma(t)}{\mathbf{1}\{z_{t} = z_{t'}\}}} - \log Z(\beta)
$$
where $\Theta = \{\lambda_{-1}, \lambda_{1}\}$, $\Gamma(t)$ is the collection of census tracts that are spatially adjacent to census tract $t$ and $\log Z(\beta)$ is the log normalizer for the latent field $p(\zseqone)$.The potential $\mathbf{1}\{z_{t} = z_{t'}\}$ expresses prior belief that adjacent census tracts should be in the same latent class. 
The connection strength $\beta$ is treated as a hyper-parameter. 
For each $\beta$ fixed, $\param$ is estimated using expectation maximization \cite{dempster1977maximum} on $\sum_{\zseqone} \log{p(\xseqone, \zseqone; \param)}$.  M-step computation is analytical, given the posteriors $p(\zseqone_t|\xseqone;\param)$. 
Exact E-step computation is reasonably efficient through smart variable elimination \cite[Chapter 9]{koller2009probabilistic}: the number of states is small and common heuristics to find good elimination orderings, such as MinFill, worked well. 
This efficient variable elimination order is also used to implement the \textsc{WeightedMarg} routine of ~\ref{alg:main_algorithm}.


\section{Additional experiments}
\label{app:extra_experiments}
We present additional experimental validation in support of the ACV methods in this section.
\subsection{Motion capture analysis}
\label{subsec:mocap}
\paragraph{Data.} We analyze motion capture recordings from the CMU MoCap database (http://mocap.cs.cmu.edu), which consists of several recordings of subjects performing a shared set of activities. We focus on the $124$ sequences from the ``Physical activities and Sports'' category that has been previously been studied~\citep{fox2009, hughes2012, fox2014} in the context of unsupervised discovery of shared  activities from the observed sequences. At each time step we retain twelve measurements deemed informative for describing the activities of interest, as recommended by~\citeauthor{fox2014}. Auto-regressive hidden Markov models have been shown effective for this task, motivating their use in this section.  

\paragraph{Accurate \lscv --- auto-regressive HMMs}
\begin{figure}[ht]
\centering
\includegraphics[width=0.64\textwidth]{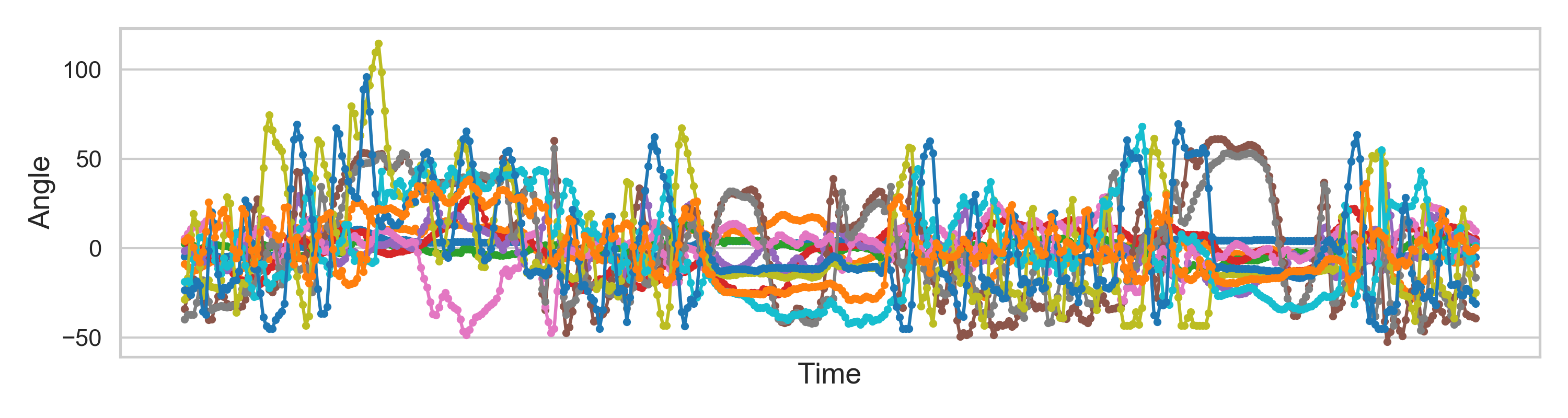}
\includegraphics[width=0.3\textwidth]{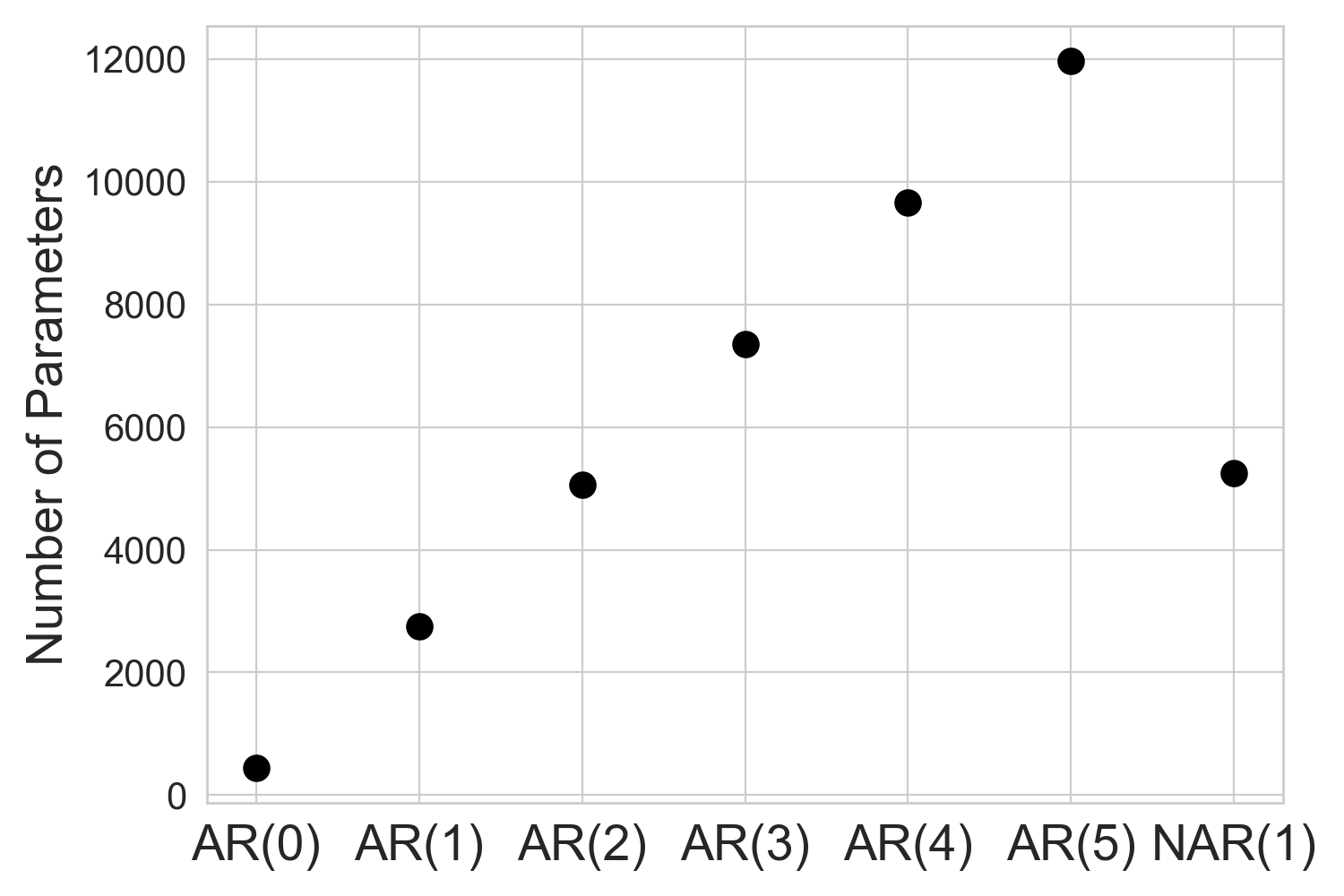} \\
\includegraphics[width=\textwidth]{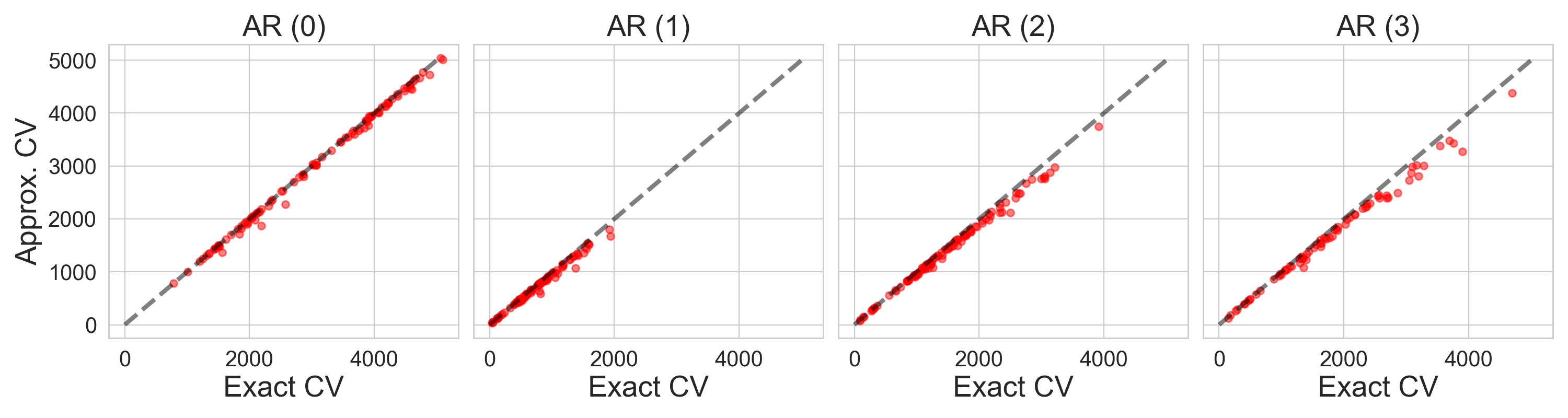}\\
\includegraphics[width=0.49\textwidth]{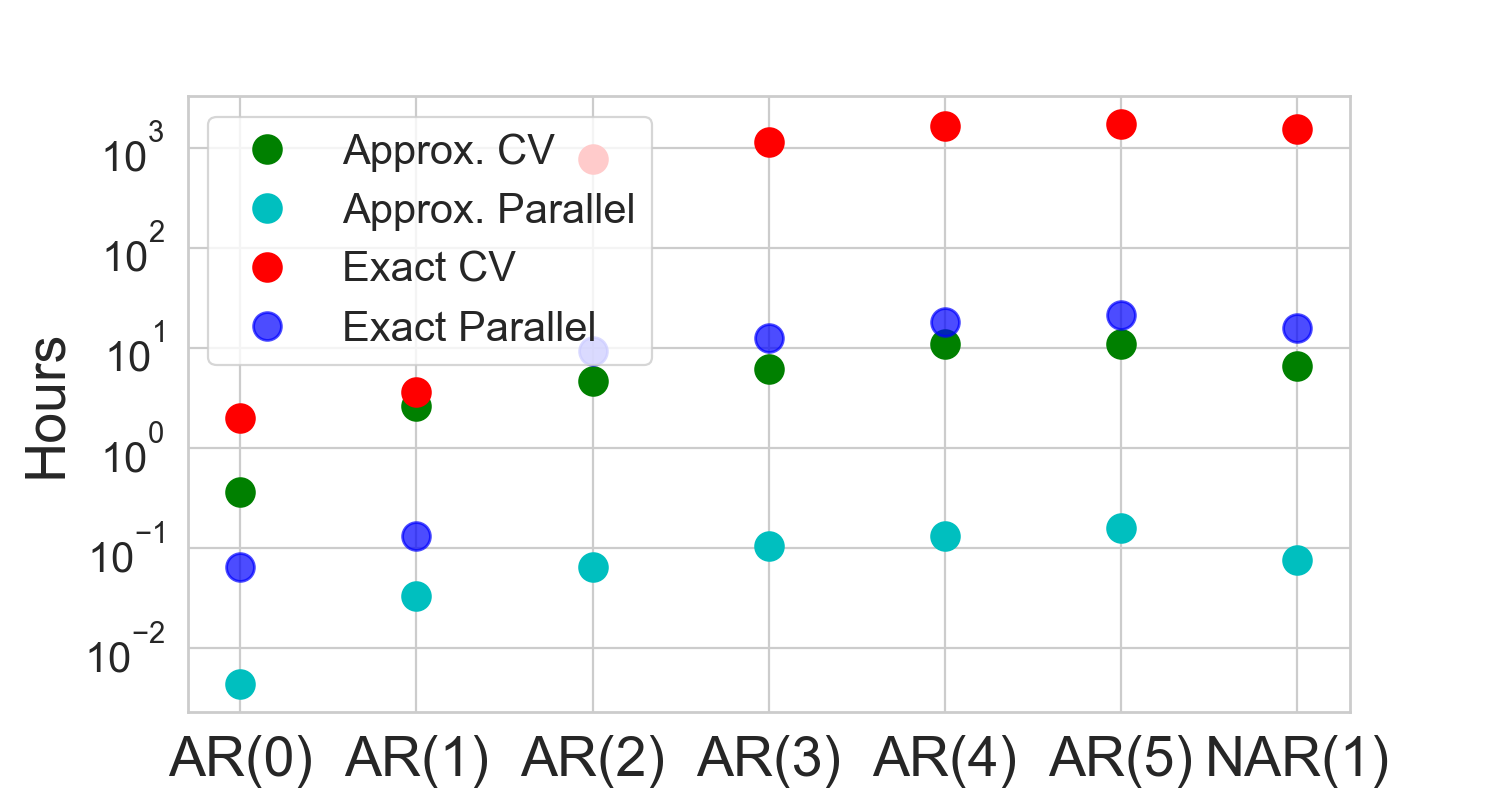}
\includegraphics[width=0.48\textwidth]{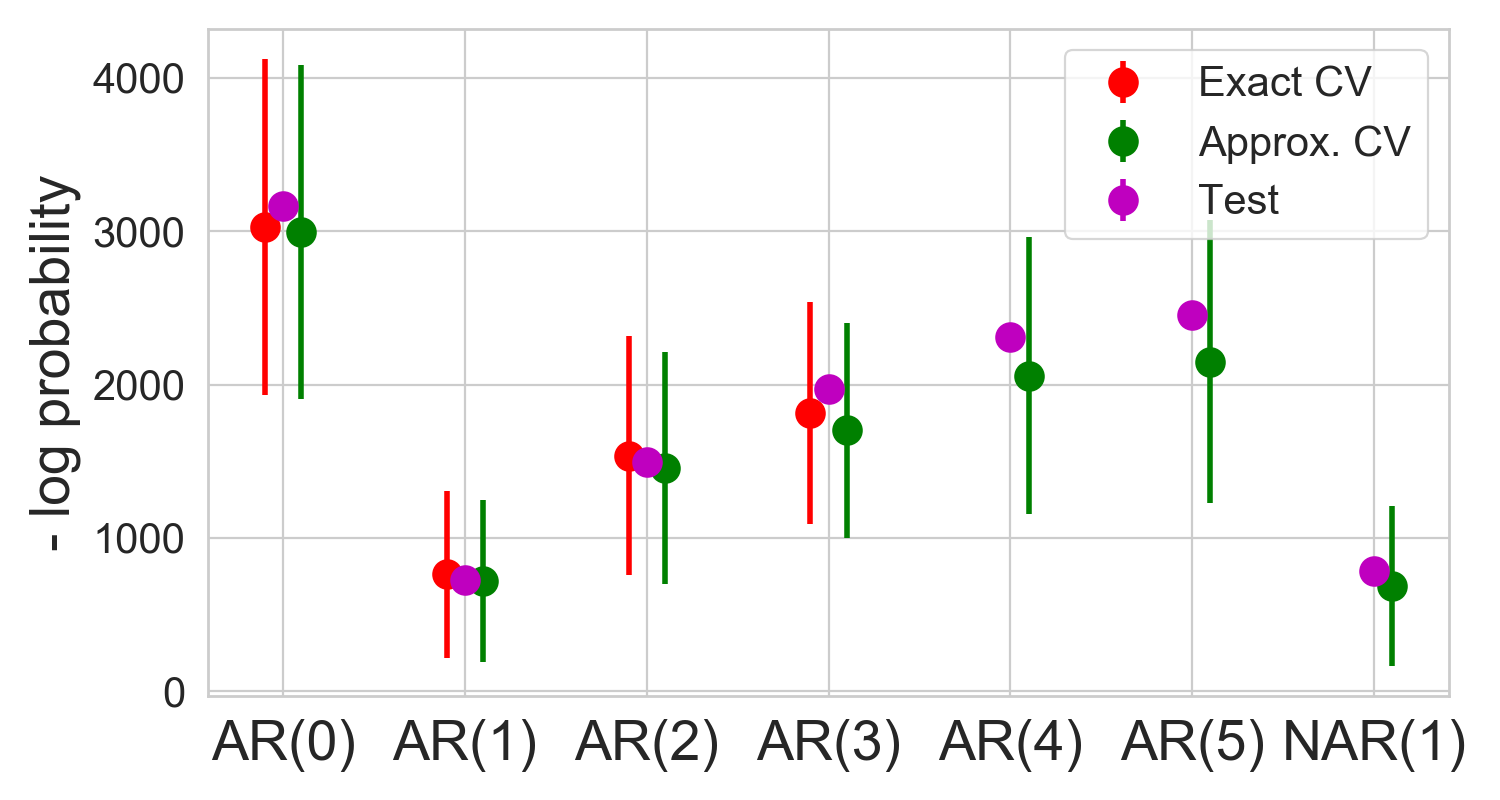}
\caption{Motion capture analysis through auto-regressive HMMs. \emph{(Top)} A twelve dimensional MoCap sequence that serves as the observed data and the number of parameters $D$ for different models under consideration. The high dimensionality of the models make alternate ACV methods based on a single Newton step infeasible.\emph{(Middle)} Scatter plots comparing leave one out loss, where x-axis is $-\ln p(\xseq \mid \param{(\wseq{\{n\}}}))$ and y-axis is $-\ln p(\xseq \mid \paramij{(\wseq{\{n\}}}))$ for different auto-regressive orders under exact and IJ approximated leave one out cross validation. Points along the diagonal indicate accurate IJ approximations. \emph{(Bottom)}  Timing and held out negative log probability across different models. For IJ and Exact the error bars represent two jackknife standard error. The IJ approximations are significantly faster but closely approximate exact leave one out loss across models and track well with test loss computed on the held out $20 \%$ of the dataset.}
\label{fig:supp_mocap}
\end{figure}
We confirm here that ACV is accurate and computationally efficient for structured models in the case studied by previous work: \lscv \emph{with} exact model fits. We present comparisons between embarrassingly parallel exact CV and \lscv with parallelized Hessian computation (``Approx. Parallel'', i.e., we parallelize the Hessian computation over different structures $n$), alleviating the primary computational bottleneck for ACV. We model the collection of MoCAP sequences via a \emph{K}-state HMM with an order-p auto-regressive (AR(p)) observation model. We also consider variants where each state's auto-regressive model is parameterized via a neural network. Figure~\ref{fig:supp_mocap} visualizes a MoCAP sequence where we have retained only the $12$ relevant dimensions.  For this experiment, we retain up to $100 (=T)$ measurements per sequence. We employ the following auto-regressive observation model,
\begin{equation}
\begin{split}
    p(\xseqt{t} \mid \xseqt{t-1}, &\ldots, \xseqt{t-p}, \zseqt{t}) = \normal(\xseqt{t} \mid \sum_{m=1}^p B_{\zseqt{t}}\xseqt{t-m} + b_{\zseqt{t}}, \sigma^2\eye), \\
    B_k &\sim \text{Matrix-Norm}(\eye, \eye, \eye), \;\; b_k\sim \normal{(0, \eye)} \quad \forall k \in \{1, \ldots, K\},
\end{split}
\end{equation}
where $p$ is the order of the auto-regression. Neural auto-regressive observation models are defined as, 
\begin{equation}
\begin{split}
   p(\xseqt{t} \mid \xseqt{t-1}, \ldots, \xseqt{t-p}, \zseqt{t}) &= \normal(\xseqt{t} \mid B^1_{\zseqt{t}}h(\sum_{m=1}^p B^0_{\zseqt{t}}\xseqt{t-m} + b^0_{\zseqt{t}}) + b^1_{\zseqt{t}}, \sigma^2\eye), \\
    \theta_k &\sim \normal(0, \lambda\eye), \quad \forall k \in \{1, \ldots, K\},
    \end{split}
\end{equation}
where $\theta_k = \{B^0_k, b^0_k, B^1_k, b^1_k\}$, and $h$ denotes a tanh non-linearity, and $B^0_k, B^0_k \in \mathbb{R}^{12 \times 12}$ and $b^0_k, b^1_k \in \mathbb{R}$, i.e., a 12-12-12 fully connected network.

While past work has explored AR(0) and AR(1) observation models, a thorough exploration of the effect of $p$ has been lacking. 
ACV provides an effective tool for exploring such questions accurately and inexpensively. We split the sequences into a $80/20 \%$ train and test split and perform \lscv on the training data $(N = 100)$ to compare AR(p) models with $p$ ranging from zero through five and the neural variant with $p=1$ (NAR(1)), in terms of how well they describe the left out sequence. 
Following~\citeauthor{fox2014}, we fix $K = 16.$
Figure~\ref{fig:supp_mocap} summarizes our results. First, 
we see that the ACV loss is quite close to the exact CV loss and that both track well with the held-out test loss. Furthermore, consistent with previous studies, we find that using an AR(1) observation model is significantly better than using an AR(0) or higher-order AR model. Interestingly, the out-of-sample loss for the AR(1) model is comparable to neural variant, NAR(1). 

In terms of computation, the ACV is significantly faster than exact CV. In fact, for the higher order auto-regressive likelihoods and the neural variant, exact CV was too expensive to perform. Instead, we report estimated time for running such experiments by multiplying the average time taken to run three folds of \lscv with the number of training instances. For AR(0) and AR(1) we compare against exact CV implemented via publicly available optimized Expectation Maximization code~\citep{bnpy}. The higher order AR and the NAR(1) model, were fit by BFGS as implemented in 
	\textsc{scipy.optimize.minimize}. We find that computing the embarrassingly parallel version provides significant speedups over their serial counterparts.

\paragraph{Accurate \lwcv for MoCAP}
\begin{figure}[ht]
\includegraphics[width=1\textwidth]{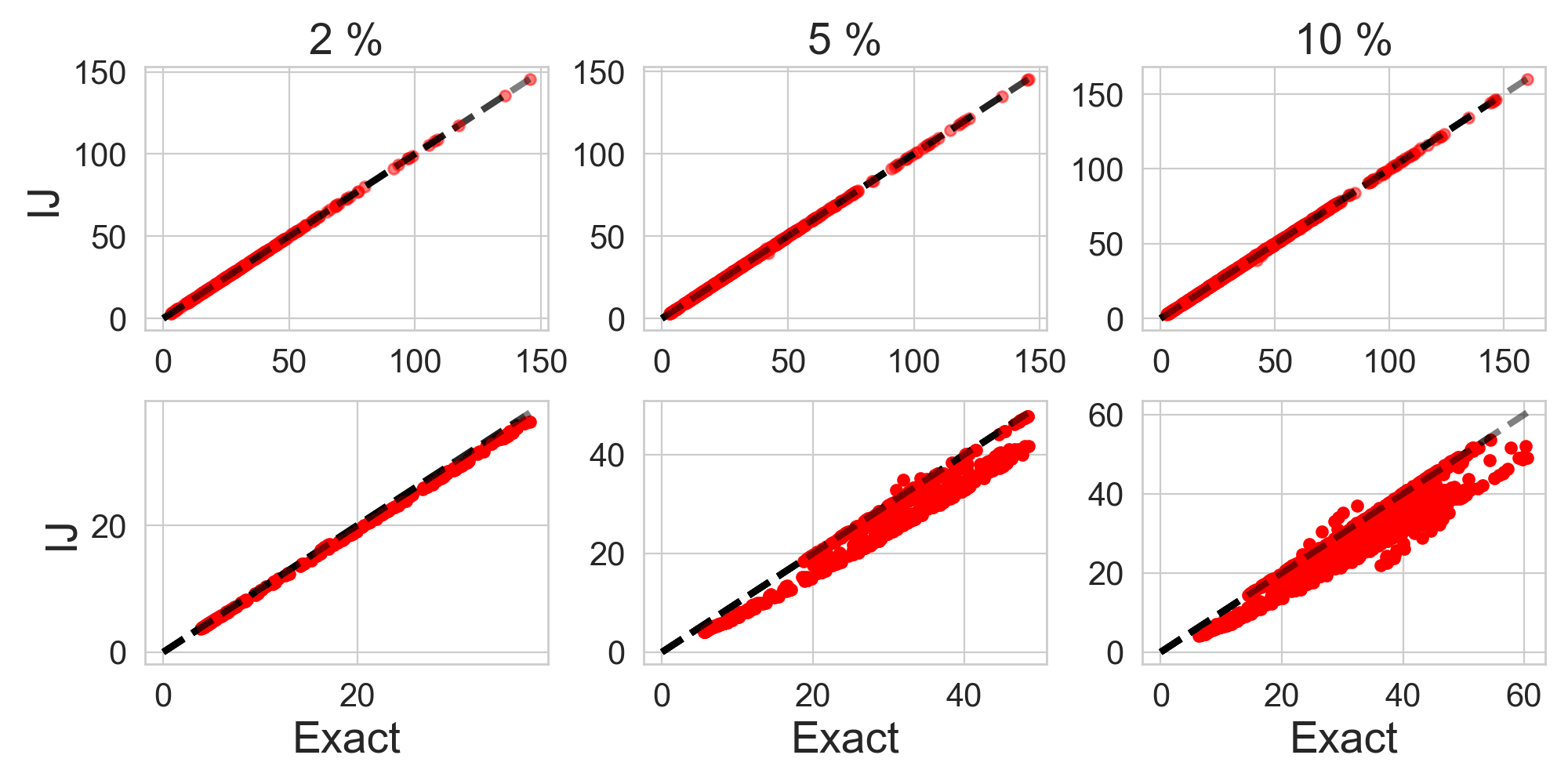}
\caption{Within sequence leave out experiments. We took the longest MoCAP sequence containing $1484$ measurements and fit a five state HMM with Gaussian emissions. We find that even for the MoCAP data IJ approximations to i.i.d. \lwcv is very accurate. As the contiguous \lwcv involves making larger scale changes to the sequence, for instance at $10\%$ we end up dropping chunks of $~ 140$ time steps from the sequences, resulting in larger changes to the parameters, IJ approximations are relatively less accurate. \emph{(Top)} Scatter plots comparing i.i.d \lwcv loss $-\ln p(x_t \mid \xobs; \param{(\wseq{\mathbf{o}})})$ (horizontal axis) with $-\ln p(x_t \mid \xobs; \paramij({\wseq{\mathbf{o}}}))$ (vertical axis), for each point $x_t$ left out in each fold, computed under exact CV for different omission rates $m\% = 2\%, 5\%,$ and $10\%$ on $M=10$ trials. \emph{(Bottom)} Results for contiguous \lwcv. }
\label{fig:appendix_withinseq}
\end{figure}
\begin{figure}[ht]
\includegraphics[width=\textwidth]{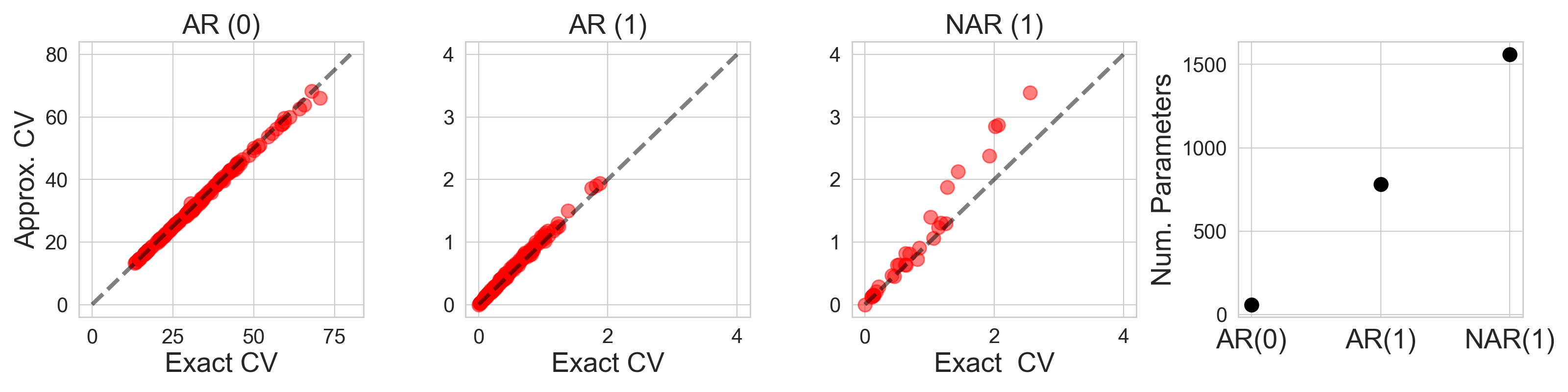}
\caption{Leave Future Out CV for MoCAP data on a single MoCAP sequence containing 1484 measurements. The scatter plots compare $-\ln p(x_{T'} \mid \xobs; \param{(\wseq{\mathbf{o}})})$ (horizontal axis)) with $-\ln p(x_{T'} \mid \xobs; \paramij({\wseq{\mathbf{o}}}))$ (vertical axis), with $\mathbf{o} = \{T', T'+1, \ldots T\}$, for some $T' \leq T$, for a five state HMM with Gaussian emissions (left), order $1$ auto-regressive emissions (middle),  neural auto-regressive emissions (right). The rightmost plot shows the number of parameters in each model. We vary $T'$ from $1337$ to $1484$ for  Gaussian and AR(1) emissions. Since exact fits the NAR model are more expensive we only vary $T'$ between $1455$ and $1484$ for NAR(1). We find that ACV to be accurate. The NAR model which is an instance of a higher dimensional optimization problem, leads to approximations that are less accurate than the lower dimensional AR(0) and AR(1) cases.}
\label{fig:appendix_LFO}
\end{figure}
Next, we present \lwcv results on a $1{,}484$ measurement long sequence extracted from the MoCAP dataset. We explore three variants of \lwcv: i.i.d \lwcv, contiguous \lwcv, and a special case of contiguous \lwcv: leave-future-out CV. Figures~\ref{fig:appendix_withinseq} and \ref{fig:appendix_LFO} present these results. We find that the IJ approximations again provide accurate approximations to exact CV. The performance deteriorates for contiguous \lwcv when large chunks of the sequence are left out. Since large changes to the sequence result in large changes to the fit parameters, a Taylor series approximation about the original fit is less accurate. Also, for high dimensional models such as NAR(1) IJ approximations tend to be less accurate~\citep{stephenson2019sparse}, explaining the drop in LFOCV performance for the NAR(1) model.


\end{document}